\RequirePackage[svgnames]{xcolor}

\documentclass[11pt,letterpaper]{mystyle}

\usepackage[all]{hypcap}
\usepackage[svgnames]{xcolor}
\usepackage[comma,authoryear,compress]{natbib}
\usepackage{hyperref}[citecolor=lightblue]

\hypersetup{
    colorlinks = true,
    citecolor = {YaleBlue},
}

\usepackage{algorithm}
\usepackage{algorithmicx}
\usepackage{algpseudocode}
\usepackage{microtype}
\usepackage{graphicx}
\usepackage{booktabs}%
\usepackage{colortbl}%
\usepackage{arydshln}%
\usepackage{float}
\usepackage{bigstrut}

\usepackage{amsmath}
\usepackage{amssymb}  
\usepackage{bm}
\usepackage{mathtools}
\usepackage{amsthm}
\usepackage{nicefrac}
\usepackage{enumitem}
\usepackage{subcaption}
\usepackage{graphicx}
\usepackage{cleveref}
\usepackage{bxcoloremoji}
\usepackage{datetime2}
\usepackage{float}

\usepackage[utf8]{inputenc}%
\usepackage[T1]{fontenc}%
\usepackage{hyperref}%
\usepackage{url}%
\usepackage{booktabs}%
\usepackage{amsfonts}%
\usepackage{nicefrac}%
\usepackage{microtype}%
\usepackage{graphicx}
\usepackage{subcaption}%
\usepackage{wrapfig}
\usepackage{lipsum}
\usepackage{enumitem}
\usepackage{stackengine}
\usepackage[font=small,labelfont=bf]{caption}
\usepackage{color}
\usepackage{adjustbox}
\usepackage{multirow, multicol}

\usepackage{rotating}
\usepackage{makecell}

\newcommand{\x}{\mathbf{x}}

\newcommand{\cv}{\mathbf{c}}
\newcommand{\xv}{\mathbf{x}}
\newcommand{\yv}{\mathbf{y}}

\DeclareMathOperator*{\E}{\mathbb{E}}

\newcommand{\kl}{D_\text{KL}}

\def\ss#1{\scriptsize #1}

\definecolor{blanchedalmond}{rgb}{1.0, 0.92, 0.8}
\definecolor{carmine}{rgb}{0.59, 0.0, 0.09}
\definecolor{lightblue}{rgb}{0.22,0.45,0.70}%

\theoremstyle{plain}
\newtheorem{theorem}{Theorem}[section]

\newtheorem{corollary}[theorem]{Corollary}
\newtheorem{proposition}[theorem]{Proposition}

\theoremstyle{definition}

\newtheorem{assumption}{Assumption}[section]

\theoremstyle{remark}
\newtheorem{remark}{Remark}[section]

\renewcommand{\mathbf}{\boldsymbol}

\makeatletter
\def\Ddots{\mathinner{\mkern1mu\raise\p@
\vbox{\kern7\p@\hbox{.}}\mkern2mu
\raise4\p@\hbox{.}\mkern2mu\raise7\p@\hbox{.}\mkern1mu}}
\makeatother

\definecolor{amaranth}{rgb}{0.9, 0.17, 0.31}
\definecolor{antiquebrass}{rgb}{0.8, 0.58, 0.46}
\definecolor{antiquefuchsia}{rgb}{0.57, 0.36, 0.51}
\definecolor{chromeyellow}{rgb}{0.31, 0.47, 0.26}

\newcommand{\J}{\mathcal{J}}

\usepackage{pifont}%

\definecolor{hpdgreen}{rgb}{0.13,0.50,0.20}
\definecolor{basered}{rgb}{0.70,0.10,0.12}
\definecolor{caseframe}{rgb}{0.20,0.25,0.32}

\newtcolorbox{casebox}[1]{
  colback=white, colframe=caseframe, boxrule=0.7pt, arc=2pt,
  left=5pt, right=5pt, top=3pt, bottom=3pt,
  fontupper=\footnotesize,
  fonttitle=\bfseries\footnotesize, coltitle=white, colbacktitle=caseframe,
  title={#1}
}

\newcommand{\cmark}{{\color{hpdgreen}\ding{51}}}
\newcommand{\xmark}{{\color{basered}\ding{55}}}
\newcommand{\Hpd}[1]{\textcolor{hpdgreen}{#1}}
\newcommand{\Base}[1]{\textcolor{basered}{#1}}
\newcommand{\method}[4]{\par\smallskip\noindent#1{\textbf{#2}}~#3\quad\footnotesize #4}
\newcommand{\caseqa}[1]{\par\noindent #1}
\newcommand{\caseline}{\par\smallskip\noindent{\color{caseframe!45}\rule{\linewidth}{0.4pt}}\par\smallskip}
\newcommand{\casepassage}[2][Edited passage (injected as training text)]{%
  \par\noindent\colorbox{gray!12}{\parbox{\dimexpr\linewidth-2\fboxsep\relax}{%
  \footnotesize\textbf{#1:}\par #2}}}

\usepackage{xspace}
\newcommand{\name}{{HPSE}\xspace}

\definecolor{lightblue}{rgb}{0.22,0.45,0.70}%
\definecolor{Gray}{gray}{0.95}
\definecolor{Cornsilk}{rgb}{1.0, 0.97, 0.86}

\usepackage{amsmath}
\usepackage[all]{hypcap}

\title{
\LARGE
Hybrid-Policy Self-Editing for Composable Unstructured Knowledge Editing
}

\runningtitle{
Hybrid-Policy Self-Editing for Composable Unstructured Knowledge Editing
}

\author{
Tianci Liu\textsuperscript{1,3} \quad
Zihan Dong\textsuperscript{2} \quad
Tianchun Li\textsuperscript{3} \quad
Yi-Chung Chen\textsuperscript{3} \quad
Qiming Cao\textsuperscript{3} \quad
Xingchen Wang\textsuperscript{3} \quad
Shiyang Wang\textsuperscript{3} \quad
Zichen Miao\textsuperscript{3} \quad
Linjun Zhang\textsuperscript{2} \quad
Haoyu Wang\textsuperscript{4} \quad
Jing Gao\textsuperscript{3} \\
\textsuperscript{1}University of Tennessee \quad
\textsuperscript{2}Rutgers University \quad
\textsuperscript{3}Purdue University \quad
\textsuperscript{4}University at Albany
}

\begin{document}

\begin{abstract}

Large language models (LLMs) achieve remarkable performance across natural language tasks, yet they are trained on static corpora and their knowledge quickly becomes outdated in a fast-changing world. This motivates knowledge editing (KE), which updates specific knowledge in an LLM without changing unrelated others. Recent works move from structured knowledge triples toward unstructured KE (UKE), where the edit is a free-form passage that may state multiple facts at once. Nonetheless, existing editors inject such a passage yet fail to \emph{use} it: the edited model can recall the passage, but can neither answer atomic questions about its facts nor compose them into multi-hop reasoning. We attribute this missing property, which we term \emph{composability}, to editors' passive reliance on the fixed passage as the sole learning source. In response, we cast editing as a proactive self-distillation from a privileged in-context state of the same model,
which requires no external supervision.
We further reveal that
due to the novelty of the injected knowledge,
the pre-edited model's own rollouts rarely cover it, which limits the effectiveness of pure on-policy distillation.
To close this gap, we propose \name, which builds a hybrid rollout that steps in to place missing facts onto the student's own trajectory precisely where its coverage fails, while staying on-policy elsewhere.
We theoretically analyze \name's advantage over pure on-policy distillation, and empirically establish its plug-and-play improvements across four LLM backbones and two KE editors under various scenarios.

\vspace{2mm}

\textit{Keywords: Knowledge Editing, Large Language Models}

\vspace{5mm}

\coloremojicode{1F4C5} 
\textbf{Date}: \today

\coloremojicode{1F4DA} \textbf{Code \& Datasets}: \href{https://github.com/lliutianc/hpse}{https://github.com/lliutianc/hpse}

\coloremojicode{1F4E7} 
\textbf{Contact}: 
\href{mailto:tliu43@tennessee.edu}{tliu43@tennessee.edu}

\end{abstract}

\maketitle
\vspace{3mm}

\section{Introduction}
\label{sec:intro}

Large language models (LLMs)~\citep{vaswani2017attention,brown2020language} demonstrate strong generalization across a wide range of language tasks~\citep{raffel2020exploring,ji2023survey}, establishing themselves as a new foundation for modern NLP~\citep{bommasani2021opportunities,zhou2023comprehensive}.
As model sizes grow, LLMs further exhibit emergent abilities to follow natural language instructions~\citep{dong2022survey,ouyang2022training}, enabling zero-shot adaptation to unseen tasks~\citep{kojima2022large,bubeck2023sparks}.

However, real-world LLM deployment remains far from resolved unless LLMs can continually refresh and internalize new knowledge beyond their training data. In a fast-changing world, an LLM's \emph{static} knowledge quickly goes out of date, leading to factual errors or even unsafe generations~\citep{de2021editing,hartvigsen2024aging}.
Yet retraining the model to absorb each update is prohibitively expensive.
As a remedy, \emph{Knowledge Editing} (KE) has been proposed to address this by updating an LLM with a \emph{specific} piece of knowledge~\citep{wang2023knowledge,zhang2024comprehensive}.

Given a direct, well-structured statement of new knowledge, e.g., a knowledge triple naming the new owner of a recently acquired company,
KE typically seeks \emph{selective} parameter updates that keep unrelated knowledge and general capabilities intact.
To name a few, \citet{meng2023memit,fang2025alphaedit} apply closed-form updates to a few causally traced MLP layers that serve as knowledge storage~\citep{dai2021knowledge},
\citet{wang2024roselora} directly learn sparse low-rank updates. \citet{liu2025unlocking} steer selected representations with non-linear updates.
Recent works move towards the more realistic setting of \emph{unstructured} KE (UKE), where the new knowledge spans multiple aspects in free-form text.
For example, an actual acquisition announcement is more likely to be a free-form passage that states multiple entangled facts at once, such as the new parent company, the incoming CEO, and the reshaped reporting structure, with no specification of which facts a user may later query.
UKE pursues this more challenging objective: treating the announcement passage itself as the knowledge to edit~\citep{wu2024akew}.
To this end, \citet{deng2024unke} extend the edit from a few MLP layers to entire transformer blocks for larger capacity. \citet{jiang2025anyedit} decompose the passage into sequential chunks for iterative editing; and \citet{su2025muke} preserve the dependencies among chunk-wise updates.
These fine-grained designs have collectively demonstrated success in making KE \emph{precise}.

Nonetheless, the edited knowledge turns out largely \emph{useless} compared to knowledge acquired through pretraining in two aspects.
The first is the loss of \emph{decomposition}: after internalizing a long-form edit, the model can at best recall the whole passage yet cannot answer targeted questions about the individual facts it encodes~\citep{zhou2026context,wang2026fable}.
The second is the loss of \emph{composition}: having learned individual facts, the model cannot compose them for multi-hop reasoning~\citep{zhong2023mquake,zhong2025mquakerm}.
Both phenomena have been analyzed in the literature, but mostly in isolation under separate scenarios~\citep{zhang2024uncovering,yao2025cake},
despite that they co-occur and are amplified in UKE.
We use \emph{composability} to denote whether edited knowledge, like its pretrained counterpart, can function as atomic facts that are flexibly extracted and recombined, and study decomposition and composition from this unified perspective.

In particular, we introduce an \emph{untargeted} regime for UKE: a free-form knowledge statement is paired with a generic editing prompt that does not specify which facts are updated, mirroring practical requests where annotators offer only a brief summary of the new passage.
We recast two representative benchmarks into a decomposition probe~\citep{deng2024unke} and a composition probe~\citep{zhong2025mquakerm}, respectively.
We find that the evaluated methods, despite their different mechanisms, consistently fall short of composability under our regime.
We attribute this failure to their \emph{passive} reliance on the provided editing context as the sole learning source, which induces severe memorization and poor generalization~\citep{zhang2024uncovering,liu2025heto,chu2025sft}; existing augmentation-based remedies~\citep{zhou2026context,yao2025cake,wang2026fable} alleviate but do not eliminate this reliance.
In summary, composability remains an open challenge, 
necessitating a generalizable UKE method; Section~\ref{sec:problem} details these findings.
\emph{Our untargeted regime and formulation of composability offer a unified view of what UKE requires in the wild. Together with our benchmarking protocol, these constitute the first key contribution of this work.}

Given the pivotal role of the editing context, we pursue a \emph{proactive} strategy through on-policy self-distillation (OPSD): the edited model generates its own response trajectories, while a \emph{privileged} state of the same model (one that reads the new passage in context) provides token-level distillation targets~\citep{lu2025opd,song2026opd,zhao2026opsd}.
Although OPSD has proved effective for adapting behavioral patterns in post-training~\citep{chu2025sft,lu2025opd,song2026opd}, it falls short in UKE: because the injected knowledge is novel, the model's own rollouts often go off-topic, yielding limited corrective supervision~\citep{gandhi2025cognitive,yue2026does}.
We term this a \emph{coverage failure}.

To mitigate this issue, we propose \name, which asks the privileged model not only to grade the edited model's own responses, but also to \emph{step in} when a rollout deviates too far from a useful trajectory.
This strategy forms hybrid, rather than purely on-policy, rollouts with better coverage, offering stronger supervision for UKE.
We further support \name with theoretical justification.
Notably, \name is a data-centric approach: it replaces only the training signal, making no assumption about which parameters are updated or how the update is parameterized, thereby offering plug-and-play flexibility across gradient-based KE editors.
\emph{The proposed \name and our theoretical analysis constitute the second key contribution of this work.}

The remainder of this paper is organized as follows. Sections~\ref{sec:problem} and~\ref{sec:method} detail composable UKE and the proposed \name, respectively.
Section~\ref{sec:experiments} evaluates \name across various settings.
Finally, Section~\ref{sec:related_work} reviews related work, and Section~\ref{sec:conclusion} concludes the paper.

\section{Composability Challenge in Unstructured Knowledge Editing}
\label{sec:problem}

This section presents composability, a critical requirement for practical unstructured knowledge editing (UKE) that has been largely overlooked in the literature. Background on LLMs is also provided.

\subsection{Preliminaries}
\label{sec:prelim}
Given a text
$\xv = (x_1, \dots, x_n)$, where each $x_i \in \mathcal V$ is a token from vocabulary $\mathcal V$,
a large language model (LLM) parameterized by $\theta$ computes probability $\pi_\theta(\xv)$ based on the chain rule~\citep{bengio2000neural}:
\begin{align}
    \pi_\theta(\xv)
    &= \prod_{i=1}^n \pi_\theta (x_i \mid {x_1, \dots, x_{i-1}})
    \triangleq \prod_{i=1}^n \pi_\theta (x_i \mid \xv_{<i}),
    \notag
\end{align}
where $\pi_\theta(x_i \mid \xv_{<i} )$ is the predicted distribution of token $x_i$ given the previous $\xv_{<i}$.
The LLM is usually trained with maximum likelihood estimation~\citep{hochreiter1997long,sutskever2014sequence,cho2014learning,yenduri2023generative}.
To generate a sentence $\xv$, the LLM computes $\pi_\theta(x_i \mid \xv_{<i})$ and draws $x_i$ from it; then $x_i$ is combined with $\xv_{<i}$ as new inputs for future steps.
This process terminates when a special token that marks the end of the sentence is generated, or when the maximum length is reached.

\textbf{Knowledge Editing (KE)} aims to update a pre-trained LLM's specific knowledge precisely while preserving other knowledge unrelated to the update, a requirement known as \emph{locality}~\citep{wang2023knowledge,zhang2024comprehensive}.
Structured KE expresses a piece of knowledge as a triple \emph{(subject, relation, object)} and requires updating \emph{object} to a new value.
Given a natural-language pair $(\xv, \yv)$ where the editing prompt $\xv$ describes a \emph{subject} and \emph{relation}, and $\yv$ specifies the corresponding \emph{object}, KE asks the LLM to respond to $\xv$ with the edited object $\yv$. 
Built upon KE, 
\textbf{Unstructured KE (UKE)} studies a more realistic regime, 
where a free-form passage $\cv$ encodes \emph{one} or \emph{multiple} pieces of new knowledge, with no labeled knowledge triple, making it more challenging \citep{wu2024akew,yang2025mirage,jiang2025anyedit}.

\subsection{Formulating and Benchmarking UKE Composability}
\label{subsection:composability}

\textbf{Composability Definition.}
KE is useful only if the knowledge it injects can be \emph{used} by the LLM, not merely \emph{memorized} in a way that can only be repeated under the exact query used for editing~\citep{zhong2023mquake,deng2024unke,zhang2024uncovering,zhong2025mquakerm,zhou2026context}.
Crucially, when it comes to UKE, where multiple pieces of knowledge are injected at once,
a \emph{useful} edit entails the following two aspects, which we collectively refer to as \emph{composability}.

\begin{itemize}
\item
\textbf{Decomposition:}
Multiple pieces of knowledge are often encoded in the passage $\cv$ and are injected jointly by editing the LLM with $\cv$ directly~\citep{deng2024unke}. In this case, an ideal LLM should be able to decompose these facts and absorb them individually. After editing, any query about an atomic fact in the passage should be answered with the fact directly, rather than by repeating the whole passage~\citep{zhou2026context,wang2026fable}.

\item
\textbf{Composition:}
Different pieces of new knowledge are usually related. Consequently, the edited model should be able to proactively combine them for multi-hop reasoning. This is closely related to the notion of \emph{portability}~\citep{zhong2023mquake,wang2024deepedit,zhong2025mquakerm}.
\end{itemize}

\textbf{Untargeted Regime.}
We further consider \emph{how} the new knowledge is presented for editing.
In practice, human labelers may want to only summarize the passage briefly, instead of detailing which contents constitute the \emph{updated} knowledge.
To reflect this demand, we formulate UKE in an \emph{untargeted} regime: the passage $\cv$ is paired with a generic editing instruction such as ``\texttt{introduce X}'', see Appendix~\ref{app:benchmarks} for detailed examples.

\textbf{Benchmarking Composability.}
We next conduct a systematic study on UKE composability.
Guided by above principles, we transform existing KE benchmarks into better (de)composition probes.
More details on our benchmark construction can be found in Appendix~\ref{app:benchmarks}.

\begin{itemize}
\item
\textbf{Decomposition probe} uses \textbf{UnKEBench}~\citep{deng2024unke}, where each sample involves a free-form passage that contains five atomic facts.
After editing, the LLM is required to answer both the passage-level question that requires \emph{joint recall} of all facts~\citep{min2023factscore}, and the targeted questions about each individual fact, which require \emph{decomposed} recall.
Unlike the original benchmark, whose editing prompt reveals the updated facts, we summarize the \emph{editing} prompt into a more generic description, in line with our untargeted regime.
The testing queries are left unchanged.
To better measure actual knowledge injection rather than lexical repetition of the passage, we use LLM-as-Judge-based scoring, following \citet{min2023factscore,deng2024unke,yang2025mirage}.

\item
\textbf{Composition probe} recasts \textbf{MQuAKE-CF-remastered}~\citep{zhong2025mquakerm}, a cleaned version of the original~\citep{zhong2023mquake}.
Each editing request in MQuAKE involves 2--4 pieces of new knowledge, each expressed as a separate single-sentence statement. After editing, the LLM is required to answer each single-hop question \emph{individually}, and to \emph{compose} them to solve multi-hop questions.
In our benchmark, we expand each single statement into an unstructured passage following the protocol of \citet{wu2024akew}. Again, we use a generic editing prompt, with testing queries unchanged. We adopt the original evaluation pipeline from the official benchmark.
\end{itemize}

\begin{wrapfigure}{r}{0.40\textwidth}
\centering
\vspace{-2em}
\includegraphics[width=0.9\linewidth]{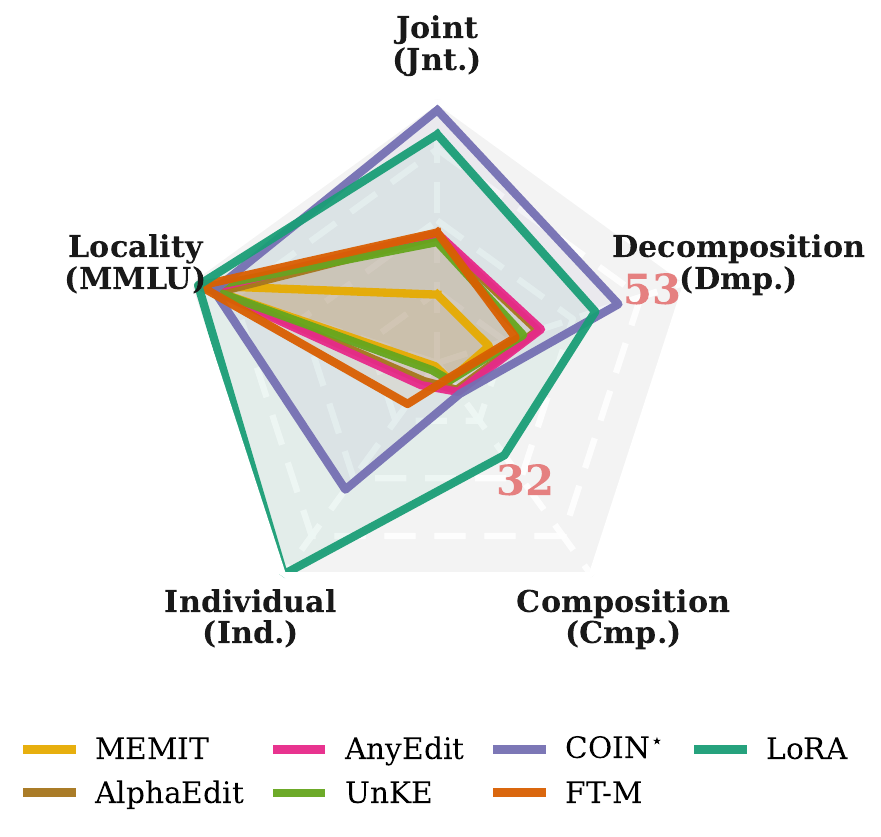}
\vspace{-1em}
\caption{
Composability results (Qwen2.5). Existing methods fail to handle (de)composition.
}
\vspace{-1.5em}
\label{fig:radar}
\end{wrapfigure}

\textbf{Benchmarking Results.}
We benchmark seven representative KE methods on editing Qwen2.5-7B-Instruct~\citep{qwen2024qwen25}, with results presented in Figure~\ref{fig:radar}: Jnt.\ and Dmp.\ denote the passage-level and fact-level accuracy of the decomposition probe, Ind.\ and Cmp.\ denote the single-hop and multi-hop accuracy of the composition probe, and Locality is measured by MMLU. The numeric results are deferred to Section~\ref{sec:experiments}.
All editors exhibit notable failures in composability.
Some editors, even those originally designed for UKE (e.g., AnyEdit~\citep{jiang2025anyedit} and UnKE~\citep{deng2024unke}), fall short of recalling the new knowledge in either a joint or a decomposed way, and the structured editors fare no better.
COIN$^{\star}$~\citep{zhou2026context}, an NTP-based editor that fine-tunes on the passage while regularizing against context reliance, achieves the highest decomposed recall (Dmp.\ 53), though largely by reproducing the whole passage (as shown later in Section~\ref{sec:experiments}), and it composes little of the injected knowledge either.
Composition also proves hard, with the strongest baseline reaching merely 32.
These failures are exemplified by cases in Figure~\ref{fig:motivation}.
On the \emph{decomposition} side, COIN$^{\star}$ reproduces the injected passage instead of answering the atomic question, carrying the target answer only indirectly; on the \emph{composition} side, AnyEdit cannot chain the two injected edits and falls back on the model's prior knowledge.

In conclusion, composability remains an open challenge for UKE, leaving a substantial gap between \emph{passively} injecting knowledge and \emph{proactively} using it.
In the next section, we propose our solution.

\begin{figure}[ht]
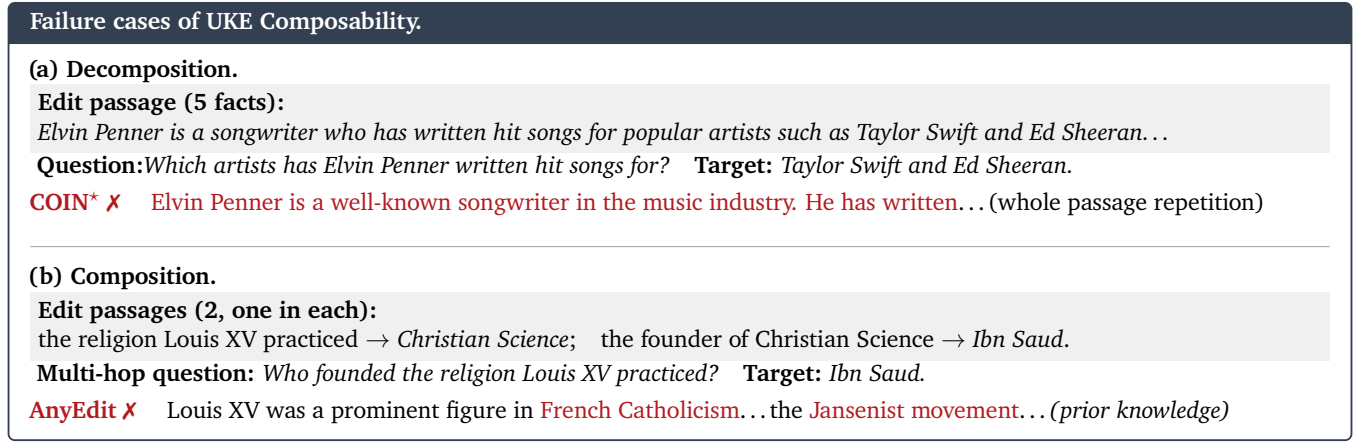

\begin{casebox}{Failure cases of UKE Composability.}
\caseqa{\textbf{(a) Decomposition.}}
\casepassage[Edit passage (5 facts)]{
    \itshape Elvin Penner is a songwriter who has written hit songs for popular artists such as Taylor Swift and Ed Sheeran\dots
}
\caseqa{
    \textbf{Question:}\emph{Which artists has Elvin Penner written hit songs for?}\quad
    \textbf{Target:} \emph{Taylor Swift and Ed Sheeran.}
}

\method{\Base}{COIN$^{\star}$}{\xmark}{
    \Base{Elvin Penner is a well-known songwriter in the music industry. He has written}\dots {(whole passage repetition)}
}
\caseline
\caseqa{\textbf{(b) Composition.}}
\casepassage[Edit passages (2, one in each)]{
the religion Louis XV practiced $\rightarrow$ \emph{Christian Science};\quad
the founder of Christian Science $\rightarrow$ \emph{Ibn Saud}.
}
\caseqa{
    \textbf{Multi-hop question:} \emph{Who founded the religion Louis XV practiced?}\quad
    \textbf{Target:} \emph{Ibn Saud.}
}

\method{\Base}{AnyEdit}{\xmark}{
    Louis XV was a prominent figure in \Base{French Catholicism}\dots the \Base{Jansenist movement}\dots\emph{(prior knowledge)}
}
\end{casebox}
\caption{
Existing KE methods fall short in handling composability, either failing to answer targeted question directly, or to reason properly with new knowledge. Errors are highlighted. 
}
\label{fig:motivation}
\end{figure}

\section{Proposed Method}
\label{sec:method}
In this section, we propose 
\name, which incorporates \emph{proactive} knowledge use at editing time toward better UKE composability. Theoretical analysis is also provided.

\subsection{Pursuing Composability via Self-Distillation}
\label{sec:method:formulation}

At a high level, the composability failures above share a common cause: minimizing $-\log \pi_\theta(\cv \mid \xv)$ supervises the new facts only under the passage's own prefixes.
Data augmentation enlarges this context set~\citep{zhou2026context,yao2025cake,wang2026fable}; however, the model remains a \emph{passive} learner from externally fixed targets, limiting its ability to generalize beyond the contexts observed during editing.
Consequently, knowledge injected in this way can remain bound to the contexts in which it was learned~\citep{zhou2026context}.

The recently emerging on-policy distillation (OPD) enables \emph{proactive} learning: rather than fitting externally specified signals, the student generates its own rollout for an external teacher to correct, an approach that has proved effective in enhancing generalization~\citep{song2026opd,li2026opd}.
Such proactivity lets the edited model invoke the new knowledge in its own generation rather than merely reproduce it under a fixed context, paving a promising path toward better composability.

However, OPD assumes a teacher that already masters the targeted knowledge.
When no such teacher is available\footnote{KE is such a case, as the knowledge to edit is by definition novel.}, on-policy self-distillation \citep[OPSD;][]{zhao2026opsd} turns the student itself into the teacher through its in-context learning capability~\citep{ike2023}.
Specifically, given the editing passage $\cv$ and editing prompt $\xv$, 
the \emph{privileged} model $\pi^\star$ is defined as the base model $\pi_0$ reading $\cv$ in context, 
and OPSD then trains the edited model $\pi_\theta$ to match $\pi^\star$ by minimizing the following loss
\begin{align}
\label{eq:opsd}
\J(\theta)
=
{\E}_{\yv \sim \pi_\theta (\cdot \mid \xv)}
\left[
\sum_{t=1}^{|\yv|}
{\kl} \big[\pi^\star(\cdot \mid \xv, \yv_{<t}) \,\|\, \pi_\theta (\cdot \mid \xv, \yv_{<t})\big]
\right],
\quad \text{where} \quad
\pi^\star (\cdot \mid \xv) \triangleq \pi_0(\cdot \mid \cv, \xv).
\end{align}
Here the forward Kullback--Leibler divergence $\kl$ matches the privileged distribution in a mass-covering manner~\citep{li2026opd}.
As $\yv$ is drawn from $\pi_\theta$ itself, the model learns to proactively correct its own output.

\begin{wrapfigure}{r}{0.42\textwidth}
\vspace{-\baselineskip}
\centering
\includegraphics[width=\linewidth]{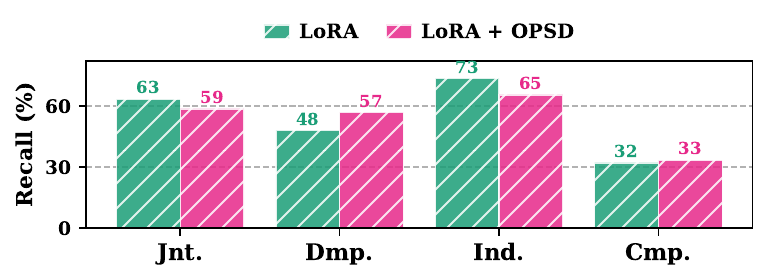}
\caption{OPSD alone brings no consistent gain.}
\label{fig:opsd_fail}
\vspace{-0.8em}
\end{wrapfigure}

\textbf{Challenge: The Coverage Failure of OPSD.}
Unfortunately, OPSD brings no consistent gain on UKE.
As shown in Figure~\ref{fig:opsd_fail}, it improves decomposed recall, yet degrades joint and individual recall, exhibiting a failure mode distinct from its post-training successes.
We provide an intuitive explanation.
Since the new knowledge in $\cv$ is by definition absent from the pre-edited model, its rollout may follow a path distant from $\cv$: the response goes \emph{off-topic}, largely irrelevant to the new facts rather than merely outdated, leaving the privileged model little room to provide corrective signal. Figure~\ref{fig:case_coverage} shows a coverage failure case.

In summary, OPSD proves insufficient to unlock composability, which motivates \name below.

\subsection{\name: A Plug-and-Play Solution}
\label{sec:objective}

Given the coverage failure in UKE, we propose \underline{H}ybrid-\underline{P}olicy \underline{S}elf-\underline{E}diting (\name) to supply self-distillation with better rollouts.
\name is a plug-and-play approach that improves the editing signal with hybrid rollout data, and thus applies to a wide range of gradient-based KE editors. 
Algorithm~\ref{alg:hpse} outlines its process.

Specifically,
\name constructs a \emph{hybrid} rollout that mixes token-level outputs from the edited model and the privileged state, which is built autoregressively as follows. 
At step $t$, given the running prefix $\yv_{<t}$, we draw $y_t$ from the hybrid rollout policy $\pi_\rho$, a per-token switch between the two policies,
\begin{align}
\label{eq:hybrid}
y_t \;\sim\; \pi_\rho(\cdot \mid \xv, \yv_{<t}) \;\triangleq\;
\begin{cases}
\;\pi^\star(\cdot \mid \xv, \yv_{<t}) & \text{if \emph{step-in} is triggered}, \\[3pt]
\;\pi_\theta(\cdot \mid \xv, \yv_{<t}) & \text{otherwise}.
\end{cases}
\end{align}
Here the \emph{step-in} is triggered when the privileged model and the student significantly disagree, and the privileged is \emph{confident} in its own prediction.
Intuitively, this repairs the off-topic rollout with a minimal change: it lets the privileged model place the missing facts onto the student's own trajectory precisely where the student would otherwise stray from $\cv$.
In practice, we trigger the \emph{step-in} if
\begin{align}
\label{eq:stepin}
\underbrace{\log \pi^\star(y^\star_t \mid \xv, \yv_{<t}) - \log \pi_\theta(y^\star_t \mid \xv, \yv_{<t})}_{\text{privileged--student gap}} > \tau
\;\;\text{and} \;\;
\underbrace{\pi^\star(y^\star_t \mid \xv, \yv_{<t})}_{\text{privileged confidence}} > \kappa, 
\end{align}
where $y^\star_t \triangleq \arg\max_{y} \pi^\star(y \mid \xv, \yv_{<t})$ and $\tau, \kappa$ are hyperparameters.
The first criterion locates the tokens the student is about to miss, and the second ensures the privileged next-token prediction is itself reliable.
Note that the confidence can also be quantified with other measures such as the privileged distribution entropy~\citep{jin2026entropyaware}. We adopt max probability due to its simplicity. 
As the KE process proceeds, the edited model injects the new knowledge, leading to fewer step-in triggers. 
Consequently, the hybrid distribution gradually converges to the pure on-policy one (see Section~\ref{sec:exp_ablation} for empirical evidence).

Building on the hybrid distribution in Eq.~\eqref{eq:hybrid}, \name trains the edited model by minimizing:
\begin{equation}
\label{eq:ours}
\begin{aligned}
\J_{\mathrm{\name}}(\theta)
&= \J_{\mathrm{hybrid}}(\theta)
   + \lambda\,\J_{\mathrm{NLL}}(\theta) \\
&= {\E}_{\yv \sim \pi_\rho (\cdot \mid \xv)}
\left[
\sum_{t=1}^{|\yv|}
{\kl} \big[\pi^\star(\cdot \mid \xv, \yv_{<t}) \,\|\,
\pi_\theta (\cdot \mid \xv, \yv_{<t})\big]
\right]
- \lambda \log \pi_\theta(\cv \mid \xv).
\end{aligned}
\end{equation}
Here the negative log-likelihood (NLL) term preserves the standard passage-level likelihood objective as a lightweight anchor.
Intuitively, the hybrid term supervises the new facts along the student's own trajectory, and the anchor grounds them under the passage's own prefixes. Section~\ref{sec:exp_ablation} ablates their respective contributions.
We set $\lambda = 1$, and draw a single greedy hybrid rollout per round for efficiency.
Finally, because \name changes the training signal without prescribing how the edit is parameterized, it can be applied to different gradient-based KE editors, as evaluated in Section~\ref{sec:experiments}.

\begin{algorithm}[H]
\caption{The proposed \name for UKE.}
\label{alg:hpse}
\begin{algorithmic}[1]
\Require passage $\cv$; editing prompt $\xv$; student model $\pi_\theta$ (initialized from $\pi_0$); privileged model $\pi^\star \triangleq \pi_0(\cdot\mid\cv,\cdot)$; thresholds $\tau,\kappa$; rounds $R$; inner steps $M$; step size $\eta$.
\For{$r = 1,\dots,R$}
  \State $\yv \gets \textsc{HybridRollout}(\pi_\theta, \pi^\star, \xv, \tau, \kappa)$ \Comment{student tokens with privileged step-in, Eq.~\eqref{eq:hybrid}}
  \For{$M$ steps}
    \State $\theta \gets \theta - \eta\,\nabla_\theta \widehat{\J}_{\mathrm{\name}}(\theta)$ \Comment{distill over the prefixes of $\yv$, Eq.~\eqref{eq:ours}}
  \EndFor
\EndFor
\State \Return $\theta$
\end{algorithmic}
\end{algorithm}

\subsection{Theoretical Analysis}
\label{sec:analysis}

We conclude this section with a theoretical analysis of \name. The informal theorem below summarizes our main theoretical result: the privileged step-in yields a fact-directed supervision advantage over pure on-policy self-distillation (OPSD). The full set of assumptions, formal version, and proof are provided in Appendix~\ref{app:theory}.

\begin{theorem}[Coverage and supervision advantage; informal]
\label{thm:main}
For an edit with a fact span of $\ell$ novel tokens, the hybrid rollout visits every fact prefix. Under sampling, an OPSD rollout reaches depth $j$ with probability at most $e^{-\tau j}$; under greedy divergence, it visits only the span entrance. Consequently, the hybrid signal is $\Omega(\ell)$ while the OPSD signal is $O(1)$, meaning that their ratio grows at least linearly with $\ell$.
\end{theorem}

Thus, the advantage of \name becomes more pronounced as the new knowledge spans more tokens, a property particularly relevant to UKE, where edits are conveyed through free-form text.
Appendix~\ref{app:theory} provides more in-depth analysis, further analyzing sampled and greedy decoding, clarifying the scope of self-termination and locality, and connecting \name to on-policy imitation learning. Under a mild context-separation condition, the same explicit per-state signal separation also holds when the NLL anchor is added to both objectives.

\section{Experiments}
\label{sec:experiments}

We evaluate the proposed \name on two KE methods applied to four language models (LMs) against
the \emph{composability} challenge.
Ablation and case studies further probe its behavior. 
Across diverse scenarios, \name consistently improves editing performance.

\subsection{Datasets and Experiment Settings}

\textbf{Base Models.}
We conduct experiments on four representative LLMs, 
\textbf{Qwen2.5-7B-Instruct} \citep{qwen2024qwen25}, \textbf{Qwen3-8B}
\citep{yang2025qwen3}, \textbf{Llama-3.1-8B-Instruct} \citep{dubey2024llama3}, and \textbf{Gemma-2-9B-it}
\citep{team2024gemma2}, which have been widely used in the literature \citep{zhang2024comprehensive,wang2024wise,deng2024unke,jiang2025anyedit,su2025muke}.
For brevity we refer to them as \textbf{Qwen2.5}, \textbf{Qwen3}, \textbf{Llama3.1},
and \textbf{Gemma2}. As is standard for instruction-tuned models, each is prompted with its own chat-template throughout.

\textbf{Tasks.}
We experiment with the two composability benchmarks detailed in Section~\ref{sec:problem}, which pose a challenging unstructured KE problem.
Namely, UnKEBench \citep{deng2024unke} targets the decomposition ability, and MQuAKE-uns the composition ability \citep{zhong2023mquake,zhong2025mquakerm}.
When editing an LLM, we consider:
(1)~\textbf{Single Editing}: one editing request is conducted at a time\footnote{A request may carry a variable number of facts, e.g., 2--4 passages per MQuAKE-uns request, which are injected as a batch.}.
(2)~\textbf{Continual Editing}: multiple editing requests are conducted sequentially, which is more demanding due to forgetting and accumulated knowledge conflict \citep{hartvigsen2024aging,wang2024wise,liu2025unlocking}.

\textbf{Editing Methods.}
We plug \name into two gradient-based KE methods \citep{xiong2025finetuning,yang2026finetuning}: \textbf{FT-M} \citep{zhang2024comprehensive}, which fine-tunes the MLP layer that causal tracing identifies as storing the knowledge, and \textbf{LoRA} \citep{hu2022lora}, which learns additive low-rank updates.
For better benchmarking, we further compare against five KE methods: \textbf{MEMIT}
\citep{meng2023memit}, \textbf{AlphaEdit} \citep{fang2025alphaedit}, \textbf{AnyEdit}
\citep{jiang2025anyedit}, \textbf{UnKE} \citep{deng2024unke}, using their official implementation; and \textbf{COIN}$^{\star}$ \citep{zhou2026context}, which was reproduced by us due to the lack of official code.
To reflect the data-scarcity regime of KE, we focus on methods that do not require large-scale, hard-to-access training data or external powerful models to help editing. This leaves augmentation-based editors \citep{yao2025cake,lee2025stepke,wang2026fable} out of scope. See Appendix~\ref{app:implementation} for implementation details.

\textbf{Evaluation Criteria.}
We evaluate the KE composability as detailed in Section~\ref{sec:problem}:
on UnKEBench, where multiple pieces of knowledge to edit are jointly encoded in a document, 
we measure the \emph{joint} recall of all edited knowledge by \textbf{Jnt.} with FActScore~\citep{min2023factscore}, and 
the \emph{decomposed} recall of each individual fact under targeted questions by \textbf{Dmp.} score.
\textbf{Div.} quantifies how distinct the model's generations are when it recalls different atomic facts.
Following \citet{wu2024akew,deng2024unke}, we report \textbf{MMLU} as a locality check.
On MQuAKE-uns, where each unstructured document only encodes a single piece of new knowledge, and multiple documents are edited in a \emph{batched} way, \textbf{Ind.} and \textbf{Cmp.} measure the LLM's individual and compositional recall, as quantified by single- and multi-hop accuracy. We further summarize each benchmark by its editing metric average. See the exact computation in Appendix~\ref{app:metrics}.

\textbf{Implementation Details.}
We use official baseline implementations; see Appendix~\ref{app:implementation} for details.

\subsection{Single Editing Performance}
\label{subsec:single}

\begin{table*}[htb!]
\centering
\caption{Single-edit performance, higher is better. COIN$^{\star}$ was reimplemented by us due to the lack of official code. Both FT-M and LoRA editors gained improvement from \name training paradigm.}
\label{tab:main}
\resizebox{\linewidth}{!}{%
\renewcommand{\tabcolsep}{3pt}
\renewcommand{\arraystretch}{0.95}
\begin{tabular}{>{\bfseries}r ccccc c ccc c ccccc c ccc}
\toprule[0.4ex]
& \multicolumn{9}{c}{\bf Qwen2.5} && \multicolumn{9}{c}{\bf Qwen3} \\
\cmidrule(lr){2-10} \cmidrule(lr){12-20}
& \multicolumn{5}{c}{\bf UnKEBench} && \multicolumn{3}{c}{\bf MQuAKE-uns} && \multicolumn{5}{c}{\bf UnKEBench} && \multicolumn{3}{c}{\bf MQuAKE-uns} \\
\cmidrule(lr){2-6} \cmidrule(lr){8-10} \cmidrule(lr){12-16} \cmidrule(lr){18-20}
& \textbf{Jnt.} & \textbf{Dmp.} & \textbf{Div.} & \textbf{Avg.} & \textbf{MMLU} && \textbf{Ind.} & \textbf{Cmp.} & \textbf{Avg.} && \textbf{Jnt.} & \textbf{Dmp.} & \textbf{Div.} & \textbf{Avg.} & \textbf{MMLU} && \textbf{Ind.} & \textbf{Cmp.} & \textbf{Avg.} \\
\midrule[0.2ex]
MEMIT & 19.2 & 15.3 & 88.3 & \textbf{40.9} & 70.3 && 1.1 & 6.0 & \textbf{3.5} && 14.7 & 16.6 & 86.1 & \textbf{39.1} & 59.1 && 1.0 & 8.0 & \textbf{4.5} \\
AlphaEdit & 35.5 & 29.6 & 86.0 & \textbf{50.4} & 64.5 && 6.3 & 9.3 & \textbf{7.8} && 28.0 & 27.3 & 85.0 & \textbf{46.8} & 30.3 && 4.2 & 6.0 & \textbf{5.1} \\
AnyEdit & 36.3 & 30.6 & 83.2 & \textbf{50.0} & 67.0 && 7.7 & 10.0 & \textbf{8.8} && 34.0 & 28.8 & 84.9 & \textbf{49.2} & 30.2 && 6.0 & 8.0 & \textbf{7.0} \\
UnKE & 34.0 & 25.3 & 87.1 & \textbf{48.8} & 68.6 && 2.5 & 6.0 & \textbf{4.2} && 32.3 & 29.5 & 84.4 & \textbf{48.7} & 32.6 && 3.0 & 8.0 & \textbf{5.5} \\
COIN$^{\star}$ & 71.0 & 53.4 & 56.5 & \textbf{60.3} & 65.9 && 43.9 & 10.7 & \textbf{27.3} && 44.0 & 35.5 & 81.6 & \textbf{53.7} & 67.5 && 21.1 & 6.7 & \textbf{13.9} \\
\noalign{\vskip 0.5ex}\cdashline{2-20}\noalign{\vskip 0.5ex}
FT-M & 36.6 & 23.1 & 84.2 & \textbf{48.0} & 70.2 && 14.2 & 5.3 & \textbf{9.7} && 20.7 & 20.5 & 86.4 & \textbf{42.5} & 67.8 && 9.4 & 5.3 & \textbf{7.3} \\
\rowcolor{gray!10}
\shortstack[r]{+ Ours \\[1pt] {\scriptsize\vphantom{+0.0\%}}} & \shortstack{43.7 \\[1pt] {\scriptsize \textcolor{teal!80!black}{+19.2\%}}} & \shortstack{30.9 \\[1pt] {\scriptsize \textcolor{teal!80!black}{+33.5\%}}} & \shortstack{87.1 \\[1pt] {\scriptsize \textcolor{teal!80!black}{+3.4\%}}} & \shortstack{\textbf{53.9} \\[1pt] {\scriptsize \textcolor{teal!80!black}{+12.2\%}}} & \shortstack{70.3 \\[1pt] {\scriptsize \textcolor{teal!80!black}{+0.2\%}}} && \shortstack{29.6 \\[1pt] {\scriptsize \textcolor{teal!80!black}{+108.5\%}}} & \shortstack{8.7 \\[1pt] {\scriptsize \textcolor{teal!80!black}{+64.2\%}}} & \shortstack{\textbf{19.1} \\[1pt] {\scriptsize \textcolor{teal!80!black}{+96.4\%}}} && \shortstack{27.0 \\[1pt] {\scriptsize \textcolor{teal!80!black}{+30.4\%}}} & \shortstack{30.0 \\[1pt] {\scriptsize \textcolor{teal!80!black}{+46.5\%}}} & \shortstack{87.6 \\[1pt] {\scriptsize \textcolor{teal!80!black}{+1.4\%}}} & \shortstack{\textbf{48.2} \\[1pt] {\scriptsize \textcolor{teal!80!black}{+13.3\%}}} & \shortstack{67.1 \\[1pt] {\scriptsize -1.1\%}} && \shortstack{17.8 \\[1pt] {\scriptsize \textcolor{teal!80!black}{+89.4\%}}} & \shortstack{6.7 \\[1pt] {\scriptsize \textcolor{teal!80!black}{+26.4\%}}} & \shortstack{\textbf{12.2} \\[1pt] {\scriptsize \textcolor{teal!80!black}{+66.7\%}}} \\
\noalign{\vskip 0.5ex}\cdashline{2-20}\noalign{\vskip 0.5ex}
LoRA & 64.1 & 46.6 & 63.2 & \textbf{58.0} & 70.6 && 73.3 & 32.0 & \textbf{52.6} && 66.3 & 56.1 & 58.6 & \textbf{60.3} & 69.0 && 68.3 & 41.3 & \textbf{54.8} \\
\rowcolor{gray!10}
\shortstack[r]{+ Ours \\[1pt] {\scriptsize\vphantom{+0.0\%}}} & \shortstack{73.3 \\[1pt] {\scriptsize \textcolor{teal!80!black}{+14.2\%}}} & \shortstack{58.5 \\[1pt] {\scriptsize \textcolor{teal!80!black}{+25.6\%}}} & \shortstack{61.5 \\[1pt] {\scriptsize -2.6\%}} & \shortstack{\textbf{64.4} \\[1pt] {\scriptsize \textcolor{teal!80!black}{+11.1\%}}} & \shortstack{70.7 \\[1pt] {\scriptsize \textcolor{teal!80!black}{+0.2\%}}} && \shortstack{83.2 \\[1pt] {\scriptsize \textcolor{teal!80!black}{+13.5\%}}} & \shortstack{54.7 \\[1pt] {\scriptsize \textcolor{teal!80!black}{+70.9\%}}} & \shortstack{\textbf{69.0} \\[1pt] {\scriptsize \textcolor{teal!80!black}{+31.0\%}}} && \shortstack{76.0 \\[1pt] {\scriptsize \textcolor{teal!80!black}{+14.7\%}}} & \shortstack{61.4 \\[1pt] {\scriptsize \textcolor{teal!80!black}{+9.5\%}}} & \shortstack{56.7 \\[1pt] {\scriptsize -3.2\%}} & \shortstack{\textbf{64.7} \\[1pt] {\scriptsize \textcolor{teal!80!black}{+7.3\%}}} & \shortstack{69.4 \\[1pt] {\scriptsize \textcolor{teal!80!black}{+0.7\%}}} && \shortstack{82.8 \\[1pt] {\scriptsize \textcolor{teal!80!black}{+21.2\%}}} & \shortstack{50.0 \\[1pt] {\scriptsize \textcolor{teal!80!black}{+21.1\%}}} & \shortstack{\textbf{66.4} \\[1pt] {\scriptsize \textcolor{teal!80!black}{+21.2\%}}} \\
\midrule[0.3ex]
& \multicolumn{9}{c}{\bf Llama3.1} && \multicolumn{9}{c}{\bf Gemma2} \\
\cmidrule(lr){2-10} \cmidrule(lr){12-20}
& \multicolumn{5}{c}{\bf UnKEBench} && \multicolumn{3}{c}{\bf MQuAKE-uns} && \multicolumn{5}{c}{\bf UnKEBench} && \multicolumn{3}{c}{\bf MQuAKE-uns} \\
\cmidrule(lr){2-6} \cmidrule(lr){8-10} \cmidrule(lr){12-16} \cmidrule(lr){18-20}
& \textbf{Jnt.} & \textbf{Dmp.} & \textbf{Div.} & \textbf{Avg.} & \textbf{MMLU} && \textbf{Ind.} & \textbf{Cmp.} & \textbf{Avg.} && \textbf{Jnt.} & \textbf{Dmp.} & \textbf{Div.} & \textbf{Avg.} & \textbf{MMLU} && \textbf{Ind.} & \textbf{Cmp.} & \textbf{Avg.} \\
\midrule[0.2ex]
MEMIT & 11.9 & 2.7 & 63.1 & \textbf{25.9} & 64.0 && 0.4 & 5.3 & \textbf{2.8} && 33.9 & 25.4 & 80.4 & \textbf{46.6} & 60.7 && 1.3 & 4.7 & \textbf{3.0} \\
AlphaEdit & 34.1 & 33.9 & 80.0 & \textbf{49.4} & 28.5 && 9.3 & 8.7 & \textbf{9.0} && 35.8 & 22.6 & 80.4 & \textbf{46.3} & 36.1 && 1.8 & 4.7 & \textbf{3.2} \\
AnyEdit & 50.9 & 46.3 & 66.2 & \textbf{54.5} & 37.4 && 16.7 & 6.7 & \textbf{11.7} && 33.0 & 28.1 & 79.9 & \textbf{47.0} & 60.9 && 3.2 & 8.7 & \textbf{5.9} \\
UnKE & 41.6 & 38.5 & 74.1 & \textbf{51.4} & 53.3 && 17.2 & 7.3 & \textbf{12.2} && 11.0 & 6.5 & 83.2 & \textbf{33.6} & 70.6 && 0.0 & 1.3 & \textbf{0.7} \\
COIN$^{\star}$ & 87.3 & 73.4 & 17.4 & \textbf{59.4} & 35.0 && 75.7 & 12.0 & \textbf{43.9} && 28.3 & 18.5 & 81.3 & \textbf{42.7} & 70.1 && 14.3 & 3.3 & \textbf{8.8} \\
\noalign{\vskip 0.5ex}\cdashline{2-20}\noalign{\vskip 0.5ex}
FT-M & 71.8 & 53.6 & 49.8 & \textbf{58.4} & 67.6 && 57.9 & 28.7 & \textbf{43.3} && 29.9 & 19.1 & 80.4 & \textbf{43.1} & 72.6 && 14.4 & 1.3 & \textbf{7.8} \\
\rowcolor{gray!10}
\shortstack[r]{+ Ours \\[1pt] {\scriptsize\vphantom{+0.0\%}}} & \shortstack{72.5 \\[1pt] {\scriptsize \textcolor{teal!80!black}{+1.0\%}}} & \shortstack{52.1 \\[1pt] {\scriptsize -2.8\%}} & \shortstack{61.1 \\[1pt] {\scriptsize \textcolor{teal!80!black}{+22.7\%}}} & \shortstack{\textbf{61.9} \\[1pt] {\scriptsize \textcolor{teal!80!black}{+6.0\%}}} & \shortstack{67.6 \\[1pt] {\scriptsize \textcolor{teal!80!black}{+0.0\%}}} && \shortstack{67.5 \\[1pt] {\scriptsize \textcolor{teal!80!black}{+16.6\%}}} & \shortstack{29.3 \\[1pt] {\scriptsize \textcolor{teal!80!black}{+2.1\%}}} & \shortstack{\textbf{48.4} \\[1pt] {\scriptsize \textcolor{teal!80!black}{+11.8\%}}} && \shortstack{36.0 \\[1pt] {\scriptsize \textcolor{teal!80!black}{+20.2\%}}} & \shortstack{23.4 \\[1pt] {\scriptsize \textcolor{teal!80!black}{+22.6\%}}} & \shortstack{84.4 \\[1pt] {\scriptsize \textcolor{teal!80!black}{+4.9\%}}} & \shortstack{\textbf{47.9} \\[1pt] {\scriptsize \textcolor{teal!80!black}{+11.0\%}}} & \shortstack{71.9 \\[1pt] {\scriptsize -1.0\%}} && \shortstack{25.6 \\[1pt] {\scriptsize \textcolor{teal!80!black}{+77.8\%}}} & \shortstack{5.3 \\[1pt] {\scriptsize \textcolor{teal!80!black}{+307.7\%}}} & \shortstack{\textbf{15.4} \\[1pt] {\scriptsize \textcolor{teal!80!black}{+96.8\%}}} \\
\noalign{\vskip 0.5ex}\cdashline{2-20}\noalign{\vskip 0.5ex}
LoRA & 76.4 & 60.0 & 48.1 & \textbf{61.5} & 65.0 && 75.8 & 46.7 & \textbf{61.3} && 85.5 & 67.8 & 32.1 & \textbf{61.8} & 72.4 && 79.1 & 41.3 & \textbf{60.2} \\
\rowcolor{gray!10}
\shortstack[r]{+ Ours \\[1pt] {\scriptsize\vphantom{+0.0\%}}} & \shortstack{75.4 \\[1pt] {\scriptsize -1.2\%}} & \shortstack{61.6 \\[1pt] {\scriptsize \textcolor{teal!80!black}{+2.8\%}}} & \shortstack{61.3 \\[1pt] {\scriptsize \textcolor{teal!80!black}{+27.3\%}}} & \shortstack{\textbf{66.1} \\[1pt] {\scriptsize \textcolor{teal!80!black}{+7.5\%}}} & \shortstack{65.5 \\[1pt] {\scriptsize \textcolor{teal!80!black}{+0.9\%}}} && \shortstack{81.0 \\[1pt] {\scriptsize \textcolor{teal!80!black}{+6.9\%}}} & \shortstack{50.0 \\[1pt] {\scriptsize \textcolor{teal!80!black}{+7.1\%}}} & \shortstack{\textbf{65.5} \\[1pt] {\scriptsize \textcolor{teal!80!black}{+6.9\%}}} && \shortstack{83.9 \\[1pt] {\scriptsize -1.9\%}} & \shortstack{66.0 \\[1pt] {\scriptsize -2.6\%}} & \shortstack{53.9 \\[1pt] {\scriptsize \textcolor{teal!80!black}{+68.1\%}}} & \shortstack{\textbf{68.0} \\[1pt] {\scriptsize \textcolor{teal!80!black}{+10.0\%}}} & \shortstack{72.0 \\[1pt] {\scriptsize -0.6\%}} && \shortstack{81.5 \\[1pt] {\scriptsize \textcolor{teal!80!black}{+3.0\%}}} & \shortstack{46.0 \\[1pt] {\scriptsize \textcolor{teal!80!black}{+11.4\%}}} & \shortstack{\textbf{63.7} \\[1pt] {\scriptsize \textcolor{teal!80!black}{+5.9\%}}} \\
\bottomrule[0.4ex]
\end{tabular}
}
\end{table*}

We first evaluated \name in the {Single Editing} setting. 
The results are reported in Table~\ref{tab:main}.

\textbf{Overall Performance Gain.}
From the table, FT-M and LoRA achieved consistent benchmark-average gains by replacing their training paradigm with \name. The improvement held across all 16 editor--LLM--benchmark combinations, with only occasional dips on individual metrics (all within two points), confirming the universality of our method.
Specifically, FT-M obtained average gains of +6.8 and +5.0 points on MQuAKE-uns and UnKEBench, respectively (relative gains averaging +67.9\% and +10.6\% across the four LLMs).
With \name, FT-M further outperformed all five baselines in 4 out of 8 cases.
LoRA showed the same trend, gaining +8.9 and +5.4 points respectively, and even beat all baselines in one case where it could not on its own.

We next examined how \name addresses the \emph{composability} requirement, with the following observations.

\textbf{\name injects decomposable knowledge.} 
According to Table~\ref{tab:main}, while COIN$^{\star}$ (on Qwen2.5 and Llama3.1) and LoRA (base version) reached high Dmp. scores by answering the targeted questions, they did so by regurgitating their lengthy Jnt. answers, as reflected by a low Div.
On the other hand, other baselines (AnyEdit, UnKE, etc.) achieved high diversity, but their low Dmp. (and Jnt.) performance indicates an editing failure.
This contextual reliance is consistent with the literature \cite{zhang2024uncovering,qi2024ice,liu2025heto,zhou2026context}.
In contrast,
\emph{\name injects the knowledge in a decomposable way}, leading to Jnt. and Dmp. improvement while largely preserving Div.
More importantly, this came \textit{without compromising locality}, as indicated by a stable MMLU score across four LLM backbones.

\textbf{\name injects composable knowledge.}
For the composition requirement, \name again delivered consistent improvements.
From the table, most baselines scored low on both Ind. and Cmp. COIN$^{\star}$ reached a high Ind. yet fell far short on Cmp. (e.g., 75.7 Ind. vs 12.0 Cmp. on Llama3.1), indicating a drastic performance gap between \emph{recalling} the new knowledge and \textit{leveraging} it for composition. These results confirm the difficulty of knowledge composition \citep{zhong2023mquake}. 
On this hard task, \name improved Cmp., by +2.3 points on average for FT-M and +9.9 points (up to +70.9\%) for LoRA, without sacrificing Ind. efficacy.

These empirical results collectively confirm the effectiveness of \name as a general \textit{plug-and-play} module that can benefit existing gradient-based KE methods toward better composability.

\subsection{Continual Editing Performance}
\label{sec:exp_seq}

\begin{figure*}[t]
\centering
\includegraphics[width=\textwidth]{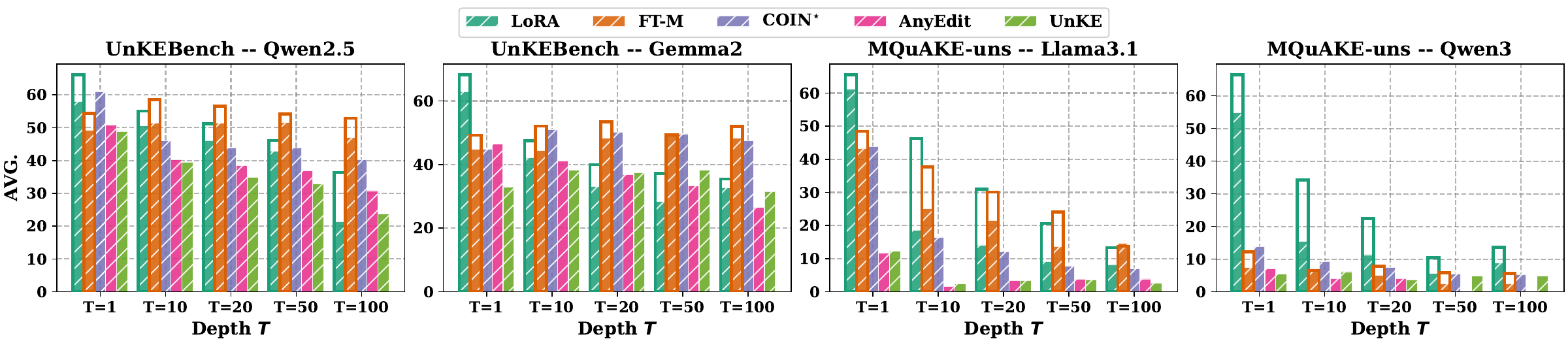}
\caption{
Continual-edit performance under different sequence length $T$, higher is better. Solid and transparent bars show performance with and
without \name. Unfilled area marks the performance gap.}
\label{fig:sequential}
\end{figure*}

We next studied the more demanding {Continual Editing} setting, 
where $T$ editing requests are injected sequentially \textit{before} the evaluation. 
Following the literature \citep{wang2024wise,liu2025unlocking}, we accumulated
$T$ edits and measured the average score.
Given the budget constraints, we evaluated each benchmark on two LLMs (Qwen2.5/Gemma2 for UnKEBench; and Llama3.1/Qwen3 for MQuAKE-uns).

Due to the page limit, we defer the complete numeric results to
Appendix~\ref{app:seq} and report the average in Figure~\ref{fig:sequential}.
For FT-M and LoRA, we used filled boxes to mark the base-version performance, and outlines drawn on top
to represent that under \name. The unfilled area quantifies \name's improvement.

As in the single-edit scenario, \name again improved the two KE editors across sequence lengths $T$ and LLMs, uniformly for LoRA and with only three exceptions for FT-M, each within one point on average. Notably, \name involves no design specific to continual editing, so this advantage reflects the robustness of its training paradigm as edits accumulated. To see this, note that on the MQuAKE-uns benchmark, LoRA improved its average score by +55\% to +149\% (relative) across all $T$ on both LLMs, more than doubling it in half of the settings.
Moreover, the gain can be large enough to make \name
\textit{competitive with the best dedicated editors}: it widened LoRA's margin over all baselines at every $T$ on Qwen3 (e.g., 10.4 vs.\ 5.5 average at $T=50$, where base LoRA led only marginally), and lifted FT-M above all baselines at 3 out of 4 horizons on Gemma2. As two different editors reached the top, this improvement was not tied to one particular editor.
To conclude, these results demonstrate the benefits of \name in diverse KE scenarios.

\subsection{Ablation Studies}
\label{sec:exp_ablation}

\begin{wraptable}{r}{0.38\textwidth}
\centering
\vspace{-2em}
\caption{{Ablation study on \name} (Qwen2.5). 
}
\label{tab:ablation}
\resizebox{\linewidth}{!}{%
\renewcommand{\tabcolsep}{3pt}
\renewcommand{\arraystretch}{1.05}
\begin{tabular}{>{\bfseries}r cccc c ccc}
\toprule[0.4ex]
& \multicolumn{4}{c}{\bf UnKEBench (subset)} && \multicolumn{3}{c}{\bf MQuAKE-uns} \\
\cmidrule(lr){2-5} \cmidrule(lr){7-9}
& \textbf{Jnt.} & \textbf{Dmp.} & \textbf{Div.} & \textbf{Avg.} && \textbf{Ind.} & \textbf{Cmp.} & \textbf{Avg.} \\
\midrule[0.2ex]
LoRA & 63.2 & 47.8 & 63.0 & \textbf{58.0} && 73.3 & 32.0 & \textbf{52.6} \\
\noalign{\vskip 0.5ex}\cdashline{2-9}\noalign{\vskip 0.5ex}
w/o both & 58.5 & 56.6 & 72.3 & \textbf{62.5} && 65.4 & 33.3 & \textbf{49.4} \\
w/o HP & 72.5 & 62.6 & 62.0 & \textbf{65.7} && 81.7 & 51.3 & \textbf{66.5} \\
w/o NLL & 74.0 & 67.1 & 60.2 & \textbf{67.1} && 82.1 & 52.0 & \textbf{67.1} \\
\noalign{\vskip 0.5ex}\cdashline{2-9}\noalign{\vskip 0.6ex}
\rowcolor{gray!10}
Ours & 75.0 & 62.5 & 60.9 & \textbf{66.1} && 83.2 & 54.7 & \textbf{68.9} \\
\bottomrule[0.4ex]
\end{tabular}
}
\vspace{-2em}
\end{wraptable}

We ablated \name to see how each component contributes to the final performance. 
We report single editing on Qwen2.5 with LoRA in Table~\ref{tab:ablation}, using a randomly selected subset.

According to the table, removing both components (``w/o both'') reduces \name to standard OPSD, which fell below even base LoRA on MQuAKE-uns (49.4 vs. 52.6 Avg.), echoing the coverage failure identified in Section~\ref{sec:method}.
Adding either component substantially improved OPSD, with the hybrid-policy (HP) rollout alone (``w/o NLL'') yielding a higher average than the NLL anchor alone (``w/o HP'') on both benchmarks. 
More importantly, with the same NLL anchor, replacing the hybrid rollout with the on-policy rollout lowered Jnt. by 2.5 points and Cmp. by 3.4 points, and this gap further widened under continual editing (to 6.3 and 6.0 points at $T=10$, see Appendix~\ref{app:ablate}).
The NLL anchor provided complementary grounding: adding it to HP improved MQuAKE-uns by 1.8 points on average, where composition depends on retaining every individually edited fact, while reducing UnKEBench by 1.0 point through lower Dmp. 
Additional sensitivity analysis on the gates $\tau$ and $\kappa$ are in  Appendix~\ref{app:sens}, where both gates exhibit notable robustness.

\begin{wrapfigure}{r}{0.35\textwidth}
\centering
\vspace{-1em}
\includegraphics[width=\linewidth]{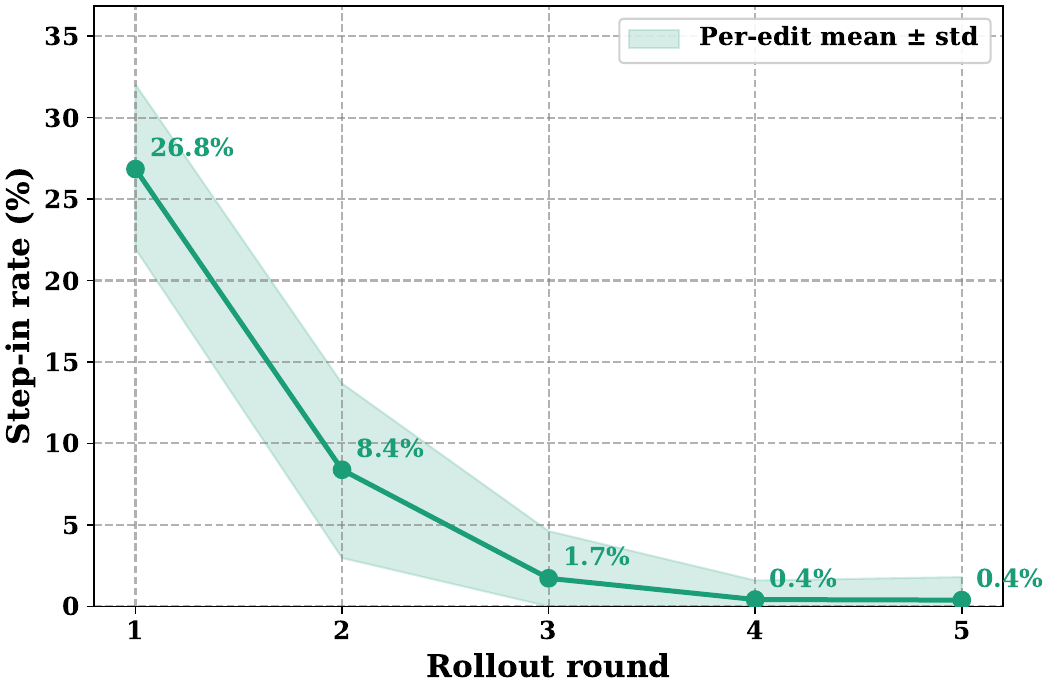}
\vspace{-1em}
\caption{{Privileged step-in dynamics} (UnKEBench, Qwen2.5). The student stops
relying on the step-in within a few rounds.}
\vspace{-2em}
\label{fig:stepin}
\end{wrapfigure}
We further examined how the privileged information enters \name's training over rounds.
To this end, we tracked the step-in frequency per sampling turn.
Results are shown in Figure~\ref{fig:stepin}.
The initial step-in rate of 26.8\% at the first outer round reflects a relatively frequent intervention, showing that the student's own trajectory initially failed to cover many facts.
More importantly, the sharp drop in the step-in rate (26.8\% to 1.7\% within 3 rounds) indicates that \name quickly internalizes the injected knowledge into the student, thereby reducing its reliance on privileged intervention as editing proceeds.
This aligns well with our analysis in Section~\ref{sec:method}.

\subsection{Case Studies}
\label{sec:exp_cases}

\begin{figure}[t]
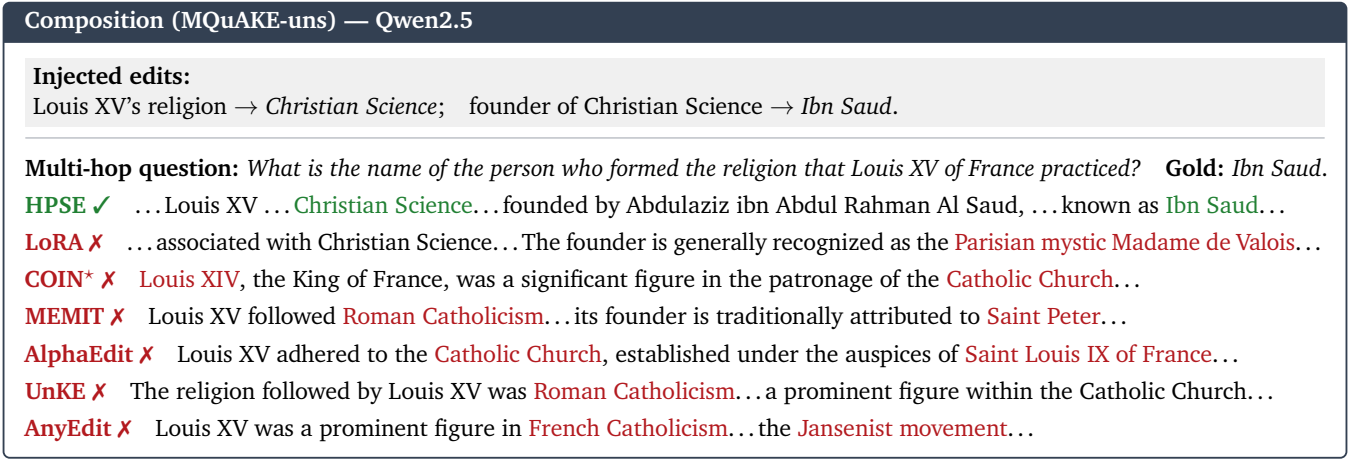

\begin{casebox}{Composition (MQuAKE-uns) \textemdash\ Qwen2.5}
\casepassage[Injected edits]{Louis XV's religion $\rightarrow$ \emph{Christian Science};\quad
founder of Christian Science $\rightarrow$ \emph{Ibn Saud}.}
\caseline
\caseqa{\textbf{Multi-hop question:} \emph{What is the name of the person who formed the religion
that Louis XV of France practiced?}\quad \textbf{Gold:} \emph{Ibn Saud}.}
\method{\Hpd}{\name}{\cmark}{\dots Louis XV \dots \Hpd{Christian Science}\dots founded by Abdulaziz ibn Abdul Rahman Al Saud, \dots known as \Hpd{Ibn Saud}\dots}
\method{\Base}{LoRA}{\xmark}{\dots associated with Christian Science\dots The founder is generally
recognized as the \Base{Parisian mystic Madame de Valois}\dots}
\method{\Base}{COIN$^{\star}$}{\xmark}{\Base{Louis XIV}, the King of France, was a significant figure in the
patronage of the \Base{Catholic Church}\dots}
\method{\Base}{MEMIT}{\xmark}{Louis XV followed \Base{Roman Catholicism}\dots its founder is
traditionally attributed to \Base{Saint Peter}\dots}
\method{\Base}{AlphaEdit}{\xmark}{Louis XV adhered to the \Base{Catholic Church}, established under
the auspices of \Base{Saint Louis IX of France}\dots}
\method{\Base}{UnKE}{\xmark}{The religion followed by Louis XV was \Base{Roman Catholicism}\dots a
prominent figure within the Catholic Church\dots}
\method{\Base}{AnyEdit}{\xmark}{Louis XV was a prominent figure in \Base{French Catholicism}\dots the
\Base{Jansenist movement}\dots}
\end{casebox}
\caption{
Composition cases with error highlighting. 
Baselines hallucinated or reverted to prior knowledge. 
}
\label{fig:case_mquake}
\end{figure}

\begin{figure}[t]
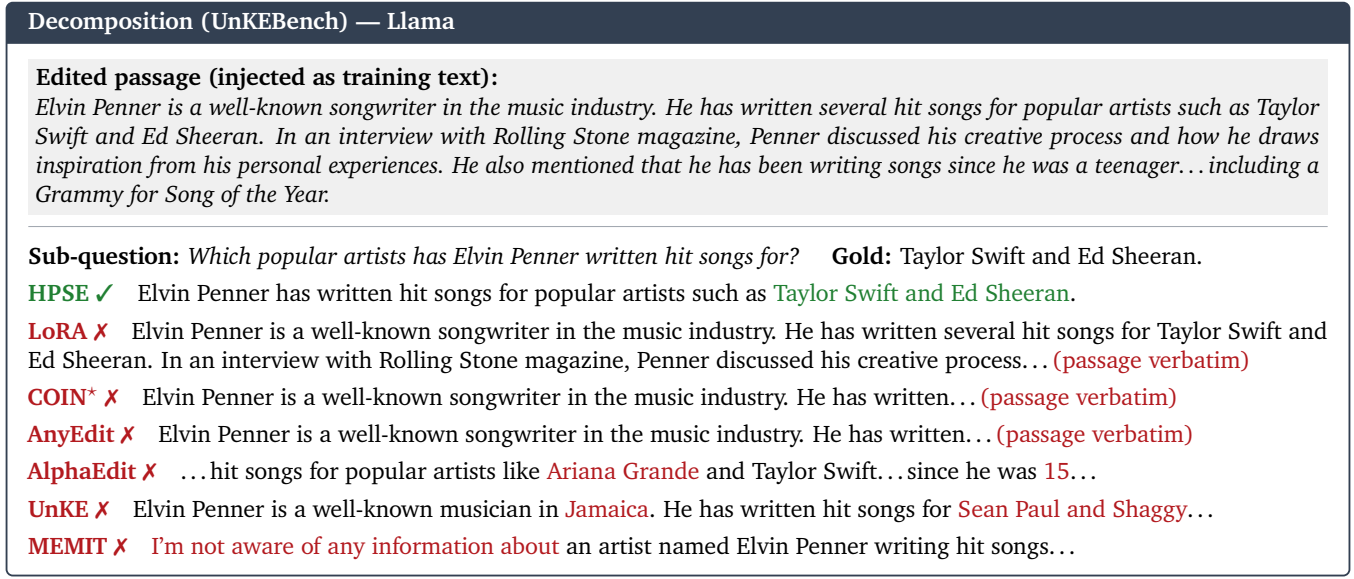

\begin{casebox}{Decomposition (UnKEBench) \textemdash\ Llama }
\casepassage{\itshape Elvin Penner is a well-known songwriter in the music industry. He has written
several hit songs for popular artists such as Taylor Swift and Ed Sheeran. In an interview with
Rolling Stone magazine, Penner discussed his creative process and how he draws inspiration from his
personal experiences. He also mentioned that he has been writing songs since he was a teenager\dots
including a Grammy for Song of the Year.}
\caseline
\caseqa{\textbf{Sub-question:} \emph{Which popular artists has Elvin Penner written hit songs for?}
\quad \textbf{Gold:} Taylor Swift and Ed Sheeran.}
\method{\Hpd}{\name}{\cmark}{Elvin Penner has written hit songs for popular artists such as
\Hpd{Taylor Swift and Ed Sheeran}.}
\method{\Base}{LoRA}{\xmark}{{Elvin Penner is a well-known songwriter in the music industry. He
has written several hit songs for} Taylor Swift and Ed Sheeran. {In an interview with Rolling
Stone magazine, Penner discussed his creative process}\dots \Base{(passage verbatim)}}
\method{\Base}{COIN$^{\star}$}{\xmark}{{Elvin Penner is a well-known songwriter in the music industry. He
has written}\dots \Base{(passage verbatim)}}
\method{\Base}{AnyEdit}{\xmark}{{Elvin Penner is a well-known songwriter in the music industry.
He has written}\dots \Base{(passage verbatim)}}
\method{\Base}{AlphaEdit}{\xmark}{\dots hit songs for popular artists like \Base{Ariana Grande} and
Taylor Swift\dots since he was \Base{15}\dots}
\method{\Base}{UnKE}{\xmark}{Elvin Penner is a well-known musician in \Base{Jamaica}. He has written
hit songs for \Base{Sean Paul and Shaggy}\dots}
\method{\Base}{MEMIT}{\xmark}{\Base{I'm not aware of any information about} an artist named Elvin
Penner writing hit songs\dots}
\end{casebox}
\caption{
Decomposition cases with error highlighting. 
Baselines repeated, hallucinated, or reverted to prior knowledge.
}
\label{fig:case_unke}
\end{figure}

We conclude this section with case studies on how \name succeeded at composable KE whereas baselines failed.
Due to the page limit, we defer further examples to Appendix~\ref{app:cases}. 

Figure~\ref{fig:case_mquake} shows a \textbf{composition} case: two edits injected via separate \emph{untargeted} passages that must be combined to answer a \emph{targeted} two-hop question.
Most baselines failed to even recall the first hop. 
Base LoRA passed this step, but failed to compose and hallucinated.
With \name, it succeeded in chaining the two.

Figure~\ref{fig:case_unke} shows a \textbf{decomposition} case, where a passage encoding five atomic facts is injected and then probed with a specific atomic question.
The baselines failed in three different ways, either repeating the passage \emph{verbatim} (indicated by low Div. as exemplified in Table~\ref{tab:main}), hallucinating, or falling back to the pre-edited knowledge.
In contrast, \name decomposed the edited knowledge and answered the sub-question directly. 
These cases demonstrate how \name injects new knowledge that is at once composable and decomposable.

\section{Related Works}
\label{sec:related_work}

Existing KE methods broadly fall into two storage paradigms \citep{zhang2024comprehensive,wang2023knowledge}.

\textbf{Internal storage KE} writes the new knowledge directly into the model parameters.
The {locate-then-edit} (LTE) paradigm~\citep{meng2022locating, meng2022mass,li2024pmet,gupta2024unified,hu2024wilke,gu2024model,fang2025alphaedit,ma2025perturbation} first localizes the weights responsible for a fact and then applies a targeted update to them.
Alternatively, PEFT methods such as LoRA \citep{hu2022lora, wang2024roselora} and ReFT \citep{wu2024reft,liu2025unlocking} edit competitively by training a small set of additional parameters.
Recent efforts extend these ideas to unstructured KE (UKE, \citet{wu2024akew}), where the edit is a document-level passage rather than a curated triple. 
To this end, \citet{jiang2025anyedit,su2025muke, zhou2026context} chunk the document and then edit them one by one~\citep{jiang2025anyedit};
and \citet{deng2024unke} edit a larger set of model parameters for better capacities.
Across these lines, the common focus is on isolating a small, knowledge-relevant subset of weights to update.
Yet they primarily supervise the edit passively with a single fixed statement, which induces severe overfitting~\citep{zhang2024uncovering,ma2024neighboring,qi2024ice,lampinen2025generalization,liu2025heto,zhou2026context} and undermines composability of UKE.
In contrast, \name provides a data-centric alternative: it proactively constructs its own training signal rather than passively fitting a fixed statement.
\name makes no assumption about which parameters are updated, complementing these editors and integrating with them seamlessly while preserving their locality.

\textbf{External storage KE}
instead holds the edit in an auxiliary memory, leaving the base parameters frozen, via meta-learning~\citep{mitchell2021fast,tan2024massive,zhang2024instructedit,li2025reinforced},
retrieval-augmented generation~\citep{zhang2024instructedit,jiang2024learning,wei2024stable,wang2024deepedit,chen2025lifelong},
and routing over separately learned weights or module copies~\citep{de2021editing,mitchell2022memory,hartvigsen2024aging,wang2024wise,yu2024melo}.
Recent works extend these ideas to UKE by augmenting the editing with atomic facts~\citep{wang2026fable}. 
However, these methods rely on large, often \textit{hard-to-access} datasets to build the retrieval store or to train the auxiliary models, which limits their applicability~\citep{wang2024wise}.
By contrast, \name is augmentation-free: it improves composability without any external data or auxiliary model, offering a self-contained and affordable path to practical KE.

\textbf{On-policy Distillation} and its self-distillation variant (OPSD) have emerged as an efficient and generalizable paradigm for LLM post-training~\citep{agarwal2024gkd,gu2024minillm,lu2025opd,song2026opd,li2026opd,zhao2026opsd}. By supervising the student on its own rollouts, OP(S)D receives dense, token-level feedback from the teacher~\citep{shenfeld2026sd,xu2026beyond}, with follow-ups strengthening the reliability of the teacher's signal~\citep{ye2026opcd,jin2026entropyaware,yu2026dopd}. Yet its gains lie mainly in reshaping behaviors the model can already produce, rather than installing genuinely new knowledge~\citep{gandhi2025cognitive,he2025rewarding,yue2026does}. Several recent works tackle related coverage issues: HDPO~\citep{ding2026hdpo} replaces entire rollouts on unsolvable RL prompts, filtered by verifiable rewards; CODE~\citep{li2026code} brings OPSD to KE, but conditions its privileged teacher on causal narratives synthesized by an external frontier model and focuses on structured KE; SKD~\citep{xu2025speculative} interleaves teacher tokens into student sampling for general-purpose distillation, using an external teacher to replace student-proposed tokens that fall outside the teacher's top-$K$ set. In contrast, \name performs token-level gated step-in using a privileged state of the model itself, thereby targeting composability in untargeted UKE without a verifier, an external teacher, or external synthesis.

\section{Conclusion}
\label{sec:conclusion}

In this work, we study \emph{composability} in unstructured knowledge editing (UKE), the requirement that an injected edit be both decomposed into its individual facts and composed into multi-hop reasoning.
Formulating UKE in an untargeted regime that mirrors practical editing requests, we recast two benchmarks into composability probes and show that existing editors, which passively rely on a fixed passage as their sole learning source, largely fall short of this requirement. In response, we propose \name, which casts editing as proactive self-distillation from a privileged in-context state of the same model and repairs the coverage gap of pure on-policy distillation through a hybrid rollout that places missing facts onto the student's own trajectory.
\name provides a plug-and-play improvement for existing gradient-based KE editors, equipping them to inject knowledge that is both decomposable and composable. Extensive experiments demonstrate the effectiveness of our method.
In the future, we plan to generalize \name toward lifelong and multimodal editing, enabling it to encode accumulated edits rather than a single passage, and to operate beyond the text-token space, thereby paving the way toward more practical UKE.

\bibliography{main}

\clearpage
\appendix
\section{Omitted Theoretical Analysis}
\label{app:theory}

In this section we develop the analysis behind Theorem~\ref{thm:main} that was omitted in the main body due to page limits. We first introduce the notation used in the analysis, then present the main theoretical results and properties of the method, and finally connect \name to on-policy imitation learning.

\subsection{Notations}
\label{app:theory:setup}

For completeness, we first collect the notation that will be used in our analysis.

For discrete distributions $p,q$ over the vocabulary $\mathcal V$, the (forward) Kullback--Leibler divergence is
\begin{equation*}
\kl[p\|q]
\;\triangleq\;
\sum_{v\in\mathcal V}p(v)\log\frac{p(v)}{q(v)}.
\end{equation*}
This is the divergence appearing in the distillation loss of Eq.~\eqref{eq:opsd}. For scalars $a,b\in(0,1)$, the binary KL between the Bernoulli distributions $(a,1-a)$ and $(b,1-b)$ is
\begin{equation*}
d_{\mathrm{bin}}(a\|b)
\;\triangleq\;
a\log\frac{a}{b}
+(1-a)\log\frac{1-a}{1-b}.
\end{equation*}
We write $\arg\max_{v}p(v)$ for the mode of $p$, $|\mathcal V|$ for the vocabulary size, and $D_{\max}$ for a uniform upper bound on the per-token KL (Assumption~\ref{as:reg}).

Recall from Section~\ref{sec:method} the student $\pi_\theta$ and the privileged teacher
\begin{equation*}
\pi^\star(\cdot\mid\xv,\yv_{<t})
\;\triangleq\;
\pi_0(\cdot\mid\cv,\xv,\yv_{<t}).
\end{equation*}
At step $t$, define the teacher's greedy token by
\begin{equation*}
y^\star_t
\;\triangleq\;
\arg\max_{v\in\mathcal V}
\pi^\star(v\mid\xv,\yv_{<t}).
\end{equation*}
The hybrid rollout policy $\pi_\rho$ of Eq.~\eqref{eq:hybrid} emits $y^\star_t$ when both conditions of the step-in gate in Eq.~\eqref{eq:stepin} hold:
\begin{equation*}
\log\pi^\star(y^\star_t\mid\cdot)
-\log\pi_\theta(y^\star_t\mid\cdot)>\tau,
\qquad
\pi^\star(y^\star_t\mid\cdot)>\kappa.
\end{equation*}
Otherwise, it emits the student's token.

We call a run of $\ell$ consecutive positions carrying the injected knowledge a \emph{fact span}, and write $F_j$ ($j=0,\dots,\ell-1$) for the prefix whose next token is the $(j{+}1)$-th fact token $y^\star_{j+1}$, with $F=\{F_0,\dots,F_{\ell-1}\}$. The \emph{deep coverage} and \emph{fact-signal} of a rollout policy $\mu$ (introduced informally in Theorem~\ref{thm:main}) are
\begin{align}
\label{eq:coverage}
c_j(\mu) &\triangleq \Pr\nolimits_{\yv\sim\mu}\big[\,\yv\text{ reproduces }y^\star_1,\dots,y^\star_j\,\big],\\
\label{eq:factsignal}
S_\mu &\triangleq \E_{\yv\sim\mu}\Big[\,\sum\nolimits_{t:\,\yv_{<t}\in F}\kl\big[\pi^\star(\cdot\mid\xv,\yv_{<t})\,\big\|\,\pi_\theta(\cdot\mid\xv,\yv_{<t})\big]\Big].
\end{align}
Since $F_j$ is visited exactly when the first $j$ fact tokens are reproduced---an event of probability $c_j(\mu)$---the fact-signal decomposes as
\begin{align}
\label{eq:signal-decomp}
S_\mu \;=\; \sum_{j=0}^{\ell-1} c_j(\mu)\,\kl\big[\pi^\star_{F_j}\,\big\|\,\pi_{\theta,F_j}\big],
\qquad \pi^\star_{F_j}\triangleq\pi^\star(\cdot\mid\xv,F_j),\quad \pi_{\theta,F_j}\triangleq\pi_\theta(\cdot\mid\xv,F_j).
\end{align}

\subsection{Main Result: Signal Separation}
\label{app:theory:main}

Given the notation introduced above, we now present our main theoretical result. Our analysis is built upon the following three assumptions.

\begin{assumption}[Novel edit]\label{as:span}
The edit introduces a length-$\ell$ run of new-knowledge tokens $y^\star_1,\dots,y^\star_\ell$ that the pre-edit student lacks but the in-context teacher supplies. At every fact prefix $F_j$,
\begin{equation*}
\pi_\theta(y^\star_{j+1}\mid F_j)\le\rho,
\qquad
\pi^\star(y^\star_{j+1}\mid F_j) > \kappa.
\end{equation*}
Moreover, the edit is \emph{$\tau$-detectable}:
\begin{equation*}
\rho<\kappa e^{-\tau}.
\end{equation*}
\end{assumption}

\begin{remark}
These conditions describe a non-trivial edit. The student cannot yet produce the new tokens, while access to the passage makes them a confident continuation for the teacher. The teacher-side condition is the standard in-context learning premise~\citep{xie2022bayesian,lampinen2025generalization}. Together, the conditions formalize the coverage failure studied in Section~\ref{sec:problem} and Figure~\ref{fig:opsd_fail}. Step~1 derives both the gate activation and the subsequent coverage collapse. The span length $\ell$ counts only the novel tokens; any part of the fact that the student already knows is excluded.
\end{remark}

\begin{assumption}[Stationary student]\label{as:sg}
The analysis holds $\pi_\theta$ fixed during each rollout. This matches the stop-gradient inner loop of Algorithm~\ref{alg:hpse}, which draws one rollout per round and treats it as a fixed dataset for its $M$ updates. For the sampling-cost statement (will be seen in Theorem~\ref{thm:rate} later), the compared rollouts are also drawn i.i.d.
\end{assumption}

\begin{remark}
This is the pre-acquisition phase of the on-policy distillation procedure~\citep{agarwal2024gkd}. Within a batch, both $\pi_\theta$ and the gate remain fixed, so $c_j$ is a well-defined success probability for each rollout.
\end{remark}

\begin{assumption}[Regularity]\label{as:reg}
The per-token KL is uniformly bounded:
\begin{equation*}
\kl[\pi^\star_t\|\pi_{\theta,t}]
\le D_{\max}<\infty.
\end{equation*}
\end{assumption}

\begin{remark}
The boundedness condition is mild. It holds, for example, when $\pi_\theta$ has a probability floor over the teacher's support. It is needed only for upper bounds on the OPSD signal; the hybrid lower bound does not use it.
\end{remark}

We first prove Theorem~\ref{thm:main} for student rollouts sampled at temperature~$1$. Remark~\ref{rem:greedy} later extends the coverage argument to greedy decoding, and Theorem~\ref{thm:rate} examines the per-round signal and sampling cost under both decoding regimes. Throughout this section, $S_\mu$ refers only to the hybrid-KL term. The NLL anchor $\lambda\J_{\mathrm{NLL}}$ in Eq.~\eqref{eq:ours} is not included in the signal comparison. Corollary~\ref{cor:anchor} below verifies that including the anchor on both sides leaves the comparison unchanged.

We now present the formal version of Theorem~\ref{thm:main} using the notation of Appendix~\ref{app:theory:setup}.

\begin{theorem}[Signal separation; formal version of Theorem~\ref{thm:main}]
\label{thm:main-formal}
Under Assumptions~\ref{as:span}--\ref{as:reg}, consider a fact span of length $\ell$ and a student rollout sampled at temperature~$1$. Define
\begin{equation*}
d(\kappa,\tau)
\triangleq
\kappa\tau
+(1-\kappa)\log\frac{1-\kappa}{1-\kappa e^{-\tau}}
>0.
\end{equation*}
Then the student and hybrid rollout policies satisfy:
\begin{enumerate}[leftmargin=1.7em,itemsep=2pt,topsep=2pt]
\item \emph{(Coverage.)} At every depth $j\le\ell$,
\begin{equation*}
c_j(\pi_\theta)\le e^{-\tau j},
\qquad
c_j(\pi_\rho)=1,
\qquad
\frac{c_j(\pi_\rho)}{c_j(\pi_\theta)}\ge e^{\tau j}.
\end{equation*}
\item \emph{(Signal.)} The fact-signals satisfy
\begin{equation*}
S_{\pi_\theta}
\le \frac{D_{\max}}{1-e^{-\tau}}
=\Theta(1),
\qquad
S_{\pi_\rho}
\ge \ell\,d(\kappa,\tau)
=\Theta(\ell),
\end{equation*}
and consequently
\begin{equation*}
\frac{S_{\pi_\rho}}{S_{\pi_\theta}}=\Omega(\ell).
\end{equation*}
\end{enumerate}
\end{theorem}

We first establish the depth-wise coverage separation.

\begin{proof}[Proof of coverage]
Let $(\mathcal F_t)_{t\ge0}$ be the filtration generated by the rollout, and define
\begin{equation*}
A_t\triangleq\{\yv_{1:t}=y^\star_{1:t}\},
\qquad
c_t(\mu)=\Pr_\mu[A_t].
\end{equation*}
Thus $A_t$ is the event that the rollout has reproduced the fact span through depth $t$. We condition throughout on the rollout reaching the span entrance $F_0$, so $c_0=1$ for both policies, and analyze coverage \emph{within} the span.

On $A_{t-1}$, the next token equals $y^\star_t$ with conditional probability
\begin{equation*}
\pi_\theta(y^\star_t\mid F_{t-1})\le\rho
\end{equation*}
by novelty (Assumption~\ref{as:span}); off $A_{t-1}$, the event $A_t$ is impossible. Novelty also \emph{derives} the gate:
\begin{equation*}
\log\pi^\star(y^\star_t\mid F_{t-1})
-\log\pi_\theta(y^\star_t\mid F_{t-1})
\ge\log\frac{\kappa}{\rho}
>\tau,
\qquad
\pi^\star(y^\star_t\mid F_{t-1})\ge\kappa.
\end{equation*}
Hence the step-in fires at each fact token---a consequence, not an assumption. For the student ($\mu=\pi_\theta$),
\begin{equation}
\label{eq:supermg}
\E\!\big[\mathbf 1_{A_t}\mid\mathcal F_{t-1}\big]=\mathbf 1_{A_{t-1}}\,\pi_\theta(y^\star_t\mid F_{t-1})\le \rho\,\mathbf 1_{A_{t-1}}.
\end{equation}
Define
\begin{equation*}
M_t\triangleq\rho^{-t}\mathbf 1_{A_t}.
\end{equation*}
Then $M_t$ is a nonnegative supermartingale. Optional stopping (equivalently, iterating the recursion in Eq.~\eqref{eq:supermg}) gives $\E[M_j]\le M_0=1$, i.e.
\begin{equation}
\label{eq:cov-collapse}
c_j(\pi_\theta)=\E[\mathbf 1_{A_j}]\le \rho^{\,j}\le e^{-\tau j}\qquad\text{for all }j\le\ell ,
\end{equation}
the last inequality following from $\rho<\kappa e^{-\tau}\le e^{-\tau}$.

Equivalently, define the first-deviation time and reproduction depth by
\begin{equation*}
T\triangleq\inf\{t\ge1:\mathbf 1_{A_t}=0\},
\qquad
D\triangleq T-1.
\end{equation*}
Here $T$ is a stopping time because $\{T\le t\}=A_t^{c}\in\mathcal F_t$. The reproduction depth satisfies
\begin{equation*}
\Pr[D\ge j]\le\rho^{\,j},
\qquad
\E[D]\le\frac{\rho}{1-\rho}.
\end{equation*}
Thus the student's own rollout leaves the fact path within a constant number of tokens in expectation.

For the hybrid ($\mu=\pi_\rho$), the gate fires at every $F_j$ as derived above. Under greedy step-in, the teacher token $y^\star_{j+1}$ is injected deterministically and $A_\ell$ holds surely. Therefore,
\begin{equation*}
c_j(\pi_\rho)=1,
\qquad
\frac{c_j(\pi_\rho)}{c_j(\pi_\theta)}\ge e^{\tau j}.
\end{equation*}
This completes our proof.
\end{proof}

We next convert the coverage separation into a fact-signal separation.

\begin{proof}[Proof of signal]
For
\begin{equation*}
\mu\in\{\pi_\theta,\pi_\rho\},
\end{equation*}
Eq.~\eqref{eq:signal-decomp} gives
\begin{equation*}
S_\mu
=\sum_{j=0}^{\ell-1}c_j(\mu)\,\mathrm{KL}_j,
\qquad
\mathrm{KL}_j
\triangleq
\kl[\pi^\star_{F_j}\|\pi_{\theta,F_j}].
\end{equation*}
The per-prefix term $\mathrm{KL}_j$ is the same under both rollout policies, so the comparison is determined by the coverage weights $c_j(\mu)$.

Assumption~\ref{as:reg} gives $\mathrm{KL}_j\le D_{\max}$. Combining this bound with the coverage result in Eq.~\eqref{eq:cov-collapse}, we obtain
\begin{align*}
S_{\pi_\theta}
&\le D_{\max}\sum_{j=0}^{\ell-1}e^{-\tau j}
\le\frac{D_{\max}}{1-e^{-\tau}}
=\Theta(1).
\end{align*}
This bound is independent of $\ell$: the geometric coverage suppresses every deep term and the student's signal is front-loaded on the shallowest tokens. For the hybrid,
\begin{equation*}
c_j(\pi_\rho)=1
\qquad\Longrightarrow\qquad
S_{\pi_\rho}=\sum_{j=0}^{\ell-1}\mathrm{KL}_j,
\end{equation*}
so it remains to lower-bound each $\mathrm{KL}_j$.

Fix $F_j$ and set
\begin{equation*}
a\triangleq\pi^\star(y^\star_{j+1}\mid F_j),
\qquad
b\triangleq\pi_\theta(y^\star_{j+1}\mid F_j).
\end{equation*}
Novelty gives
\begin{equation*}
a\ge\kappa,
\qquad
b\le\rho<\kappa e^{-\tau}\le e^{-\tau}a.
\end{equation*}
Coarsening the vocabulary to the binary event $\{y=y^\star_{j+1}\}$, the data-processing inequality~\citep{cover1999elements} gives
\begin{equation*}
\mathrm{KL}_j\ge d_{\mathrm{bin}}(a\|b).
\end{equation*}
For $b<a$,
\begin{align*}
\partial_b d_{\mathrm{bin}}(a\|b)
&=-\frac{a}{b}+\frac{1-a}{1-b}<0,\\
d_{\mathrm{bin}}(a\|b)
&\ge d_{\mathrm{bin}}(a\|e^{-\tau}a)
\eqqcolon g(a),\\
g(a)
&=a\tau+(1-a)\log\frac{1-a}{1-e^{-\tau}a}.
\end{align*}
Differentiating and using $\log x\ge(x-1)/x$,
\begin{align*}
g'(a)
&=\tau-1+\log\frac{1-e^{-\tau}a}{1-a} \\
&\quad+\frac{(1-a)e^{-\tau}}{1-e^{-\tau}a} \\
&\ge\tau-1
+\frac{a(1-e^{-\tau})+(1-a)e^{-\tau}}{1-e^{-\tau}a} \\
&=\tau-1+\frac{a+e^{-\tau}-2ae^{-\tau}}{1-e^{-\tau}a}.
\end{align*}
Moreover,
\begin{equation*}
\frac{a+e^{-\tau}-2ae^{-\tau}}{1-e^{-\tau}a}-e^{-\tau}
=\frac{a(1-e^{-\tau})^2}{1-e^{-\tau}a}
\ge0.
\end{equation*}
Thus $g'(a)\ge\tau-1+e^{-\tau}\ge0$ for $\tau>0$; the last quantity vanishes only at $\tau=0$. Hence $g$ is nondecreasing on $(0,1)$ and
\begin{equation*}
\mathrm{KL}_j\ge g(\kappa)=d(\kappa,\tau).
\end{equation*}
Using
\begin{equation*}
\log\frac{1-\kappa}{1-\kappa e^{-\tau}}
\ge-\frac{\kappa(1-e^{-\tau})}{1-\kappa},
\end{equation*}
we further have
\begin{align*}
d(\kappa,\tau)
&=\kappa\tau+(1-\kappa)\log\frac{1-\kappa}{1-\kappa e^{-\tau}}\\
&\ge\kappa(\tau-1+e^{-\tau})>0.
\end{align*}
Summing over the span yields
\begin{equation*}
S_{\pi_\rho}\ge\ell\,d(\kappa,\tau)=\Theta(\ell),
\qquad
\frac{S_{\pi_\rho}}{S_{\pi_\theta}}
\ge
\frac{\ell\,d(\kappa,\tau)(1-e^{-\tau})}{D_{\max}}
=\Omega(\ell).
\end{equation*}
This completes our proof.
\end{proof}

\begin{remark}[Interpreting the signal separation]\label{rem:signal-interpretation}
We have two observations about our signal-separation result.

\emph{Role of the confidence gate.}
The quantity $d(\kappa,\tau)$ increases in both $\kappa$ and $\tau$, and the confidence gate is essential for a non-vanishing per-token signal. With condition~(i) alone, one may send $a,b\to0$ at a fixed ratio $a/b>e^{\tau}$, in which case
\begin{equation*}
d_{\mathrm{bin}}(a\|b)\to0.
\end{equation*}
The floor $a\ge\kappa$ rules out this degeneration and keeps the per-token signal bounded away from zero.

\emph{Asymptotic scope.}
The separation is asymptotic in $\ell$, with $\tau$, $\kappa$, and $D_{\max}$ held fixed. In particular, the OPSD bound
\begin{equation*}
\frac{D_{\max}}{1-e^{-\tau}}
\end{equation*}
itself grows as $\tau\to0$. The result therefore concerns the distinct scaling of the two signals with the fact-span length $\ell$; it is not an absolute claim that $S_{\pi_\theta}$ must be numerically small in every parameter regime.
\end{remark}

Theorem~\ref{thm:main-formal} compares the two rollout policies through the distillation term alone, whereas the full objective in Eq.~\eqref{eq:ours} also carries the NLL anchor $\lambda\J_{\mathrm{NLL}}$. We next verify that including the anchor on both sides does not alter the separation. The key condition is that the anchor supervises a different set of prefixes.

\begin{assumption}[Context separation]\label{as:sep}
The fact prefixes and the passage prefixes differ as token sequences:
\begin{equation*}
F_j\ne\cv_{<t}
\qquad\text{for all }j=0,\dots,\ell-1\text{ and }t=1,\dots,|\cv|.
\end{equation*}
\end{assumption}

\begin{remark}
The condition states that the rollout does not retrace the passage verbatim. It is natural in our untargeted regime: the response to a generic editing prompt follows the chat template and an answer-style phrasing, whereas $\cv$ is document-style text. When the condition fails, i.e., the ideal response reproduces the passage word by word, the anchor does supervise the fact prefixes. This boundary is visible at the metric level in Section~\ref{sec:exp_ablation}: adding the anchor aids joint recall, whose target retraces the passage, while lowering decomposed recall; effects beyond such passage-like responses, e.g., its gain on MQuAKE-uns, operate through parameter sharing rather than through explicit fact-prefix supervision.
\end{remark}

\begin{corollary}[Signal separation with the NLL anchor]\label{cor:anchor}
Under Assumptions~\ref{as:span}--\ref{as:reg} and~\ref{as:sep}, augment both objectives in Theorem~\ref{thm:main-formal} with the anchor, i.e., compare $\J(\theta)+\lambda\J_{\mathrm{NLL}}(\theta)$ of Eq.~\eqref{eq:opsd} against $\J_{\mathrm{\name}}(\theta)$ of Eq.~\eqref{eq:ours}. Then:
\begin{enumerate}[leftmargin=1.7em,itemsep=2pt,topsep=2pt]
\item \emph{(Zero explicit anchor mass on fact prefixes.)} The anchor places no explicit supervision term on any $F_j$. Hence the fact-signals of the anchored objectives coincide with $S_{\pi_\theta}$ and $S_{\pi_\rho}$, and the bounds of Theorem~\ref{thm:main-formal} hold verbatim:
\begin{equation*}
S_{\pi_\theta}\le\frac{D_{\max}}{1-e^{-\tau}}=\Theta(1),
\qquad
S_{\pi_\rho}\ge\ell\,d(\kappa,\tau)=\Theta(\ell).
\end{equation*}
\item \emph{(Explicit constraint at fact prefixes.)} Viewed as a functional of the per-state conditionals, the anchored OPSD objective contains no term at $F_j$ beyond the KL term of weight $c_j(\pi_\theta)\le e^{-\tau j}$, which reduces to the single entrance term ($j=0$) under the greedy hypothesis of Theorem~\ref{thm:rate}. In contrast, any zero-loss minimizer of the anchored hybrid objective satisfies $\pi_\theta(\cdot\mid F_j)=\pi^\star(\cdot\mid F_j)$ at every $F_j$.
\end{enumerate}
\end{corollary}

\begin{proof}
The anchor decomposes over passage prefixes,
\begin{equation*}
\J_{\mathrm{NLL}}(\theta)
=-\log\pi_\theta(\cv\mid\xv)
=-\sum_{t=1}^{|\cv|}\log\pi_\theta(c_t\mid\xv,\cv_{<t}),
\end{equation*}
so it depends on the conditionals of $\pi_\theta$ only at the states $(\xv,\cv_{<t})$. By Assumption~\ref{as:sep} these states differ from every $F_j$, so the anchor contributes no term to the fact-signal of Eq.~\eqref{eq:factsignal}, which sums only over prefixes in $F$. Claim~(i) then follows from Theorem~\ref{thm:main-formal}.

For claim~(ii), the fact-prefix contribution to the anchored OPSD objective is the KL term weighted by $c_j(\pi_\theta)\le e^{-\tau j}$ from Eq.~\eqref{eq:cov-collapse}, and is bounded by $e^{-\tau j}D_{\max}$ under Assumption~\ref{as:reg}; under the greedy hypothesis of Theorem~\ref{thm:rate}, only $F_0$ is visited and the contribution reduces to the single term there. The hybrid statement repeats the argument of Proposition~\ref{prop:consistency}(ii) with the anchor present: the anchor does not constrain $\pi_\theta(\cdot\mid F_j)$ by Assumption~\ref{as:sep}, while the hybrid distillation term assigns weight one to every $F_j$ and the forward KL vanishes only at equality.
This completes our proof.
\end{proof}

\begin{remark}[Scope of the anchored comparison]
The corollary treats the objectives at the level of visitation, which pins down what each loss enforces per state. In a shared-parameter network the anchor can still move the conditionals at $F_j$ through generalization across states; this is why the anchored OPSD variant remains a non-trivial competitor in Section~\ref{sec:exp_ablation}. What the anchor fits explicitly, however, is the passage-context conditional, whose context-bound nature is reflected in the lower decomposed recall (Dmp.) in Section~\ref{sec:exp_ablation}. The comparison of what the objectives enforce on the fact prefixes is therefore unchanged by the anchor; the corollary concerns this explicit per-state signal only and makes no claim about the anchor's net effect on parameters or downstream performance.
\end{remark}
\begin{remark}[Greedy vs.\ sampled]\label{rem:greedy}
The bound in Eq.~\eqref{eq:cov-collapse} concerns sampled rollouts because the conditional probabilities in Eq.~\eqref{eq:supermg} are sampling probabilities. Under greedy decoding, it instead describes expected coverage over the prompt distribution. If moreover $\tau\ge\log|\mathcal V|$, then
\begin{equation*}
\pi_\theta(y^\star_i\mid\cdot)
<e^{-\tau}
\le\frac{1}{|\mathcal V|}
\le\max_y\pi_\theta(y\mid\cdot).
\end{equation*}
Thus $y^\star_i$ is not the greedy student's argmax, and the student diverges at the first fact token:
\begin{equation*}
c_j(\pi_\theta)=0
\qquad\text{for all }j\ge1.
\end{equation*}
This is an even stronger collapse.
\end{remark}

Theorem~\ref{thm:main} characterizes the signal available within a single training round. Under a student rollout, supervision on deeper fact tokens becomes exponentially rare. We next examine how the decoding regime shapes this per-round signal and the cost of sampling deep fact prefixes.

\begin{theorem}[Effect of the decoding regime]\label{thm:rate}
Beyond the per-round signal of Theorem~\ref{thm:main}:
\begin{enumerate}[leftmargin=1.7em,itemsep=1pt,topsep=2pt]
\item \emph{(Greedy, per round.)} Under greedy decoding, suppose the edit is counterfactual at the span entrance:
\begin{equation*}
\arg\max_v\pi_\theta(v\mid F_0)\ne y^\star_1.
\end{equation*}
Then a student-only (OPSD) rollout diverges at the entrance: it visits $F_0$ but no deeper $F_j$ with $j\ge1$, whereas the hybrid injects the entire span. Their per-round fact-signals satisfy
\begin{equation*}
S_{\pi_\theta}=\kl\big[\pi^\star_{F_0}\,\big\|\,\pi_{\theta,F_0}\big]\le D_{\max}=O(1),
\qquad
S_{\pi_\rho}\ge\ell\,d(\kappa,\tau)=\Omega(\ell).
\end{equation*}
\item \emph{(Sampled, waiting time.)} Under temperature-$1$ sampling, a \emph{fixed} student first reaches depth $j$ only after $\Omega(e^{\tau j})$ draws in expectation (Eq.~\eqref{eq:cov-collapse}).
\end{enumerate}
\end{theorem}

\Needspace{8\baselineskip}
\paragraph{Greedy decoding.}
We first establish the per-round statement under greedy decoding.

\begin{proof}[Proof of the greedy statement]
Greedy decoding makes the student rollout the deterministic argmax path. Conditioned on reaching the span entrance, as assumed throughout, the rollout visits $F_0$ and deposits the single KL term there, bounded by $D_{\max}$ under Assumption~\ref{as:reg}. By hypothesis its next token differs from $y^\star_1$, so no deeper $F_j$ is visited and no further term accrues.
The hybrid injects every $y^\star_{j+1}$ by the coverage result in Theorem~\ref{thm:main-formal}(i), and its signal follows from the lower bound in Theorem~\ref{thm:main-formal}(ii). This completes our proof.
\end{proof}

\Needspace{8\baselineskip}
\paragraph{Sampled decoding.}
We next establish the learning-time statement under temperature-$1$ sampling.

\begin{proof}[Proof of the sampled statement]
For a fixed student, the draws are i.i.d.\ by Assumption~\ref{as:sg}, so the first-success (geometric) waiting time to reach depth $j$ is $1/c_j(\pi_\theta)\ge e^{\tau j}$ by Eq.~\eqref{eq:cov-collapse}.
This completes our proof.
\end{proof}

\begin{remark}[Interpreting the decoding-regime result]
We close this subsection by clarifying the scope of Theorem~\ref{thm:rate} and then noting its implication for compositional queries.

\emph{Scope.}
Part~(i) concerns the signal available in a single round. It is not a convergence-rate result and does not imply that OPSD can never learn. In a shared-parameter network, generalization across states and the NLL anchor $\lambda\J_{\mathrm{NLL}}$ can still produce slow partial progress under greedy decoding. This agrees with the nonzero gains of the on-policy ablation (``w/o HP'') in Section~\ref{sec:experiments}. The regime-independent conclusion remains the per-round signal separation in Theorem~\ref{thm:main}. A convergence-rate separation for shared-parameter networks remains open. Theorem~\ref{thm:rate} makes a narrower distinction: under greedy decoding the per-round explicit signal is $O(1)$ versus $\Omega(\ell)$, and under sampling the waiting-time bound concerns a fixed student. The latter need not compound over training: heuristically, once shallow tokens are acquired, the coverage of deeper prefixes rises toward its entrance value, suggesting a shallow-to-deep bootstrap over roughly $O(\ell/c_0)$ rounds. This heuristic presumes that supervised tokens are retained and that coverage improves monotonically as training proceeds, which we do not formalize.

\emph{Compositional implication.}
Our rollouts are drawn on the editing prompt, so compositional states are not visited during editing, and the analysis does not directly bound multi-hop performance. Its implication is instead per-fact. A $k$-hop query at test time succeeds only if every involved fact surfaces reliably, so compositional accuracy compounds the per-fact usability that the editing signal establishes. The separation above governs exactly this per-fact signal: the hybrid supervises each fact span in full within the round, whereas under a student rollout the deeper tokens of every span are starved. This is consistent with, though it does not by itself explain, the individual-versus-compositional gap in Section~\ref{sec:experiments}.
\end{remark}

\subsection{Self-Termination, Consistency, and Locality}
\label{app:theory:props}

We next study three properties of the proposed intervention: self-termination, consistency, and locality.

\begin{proposition}[Self-termination, consistency, and locality]\label{prop:consistency}
\begin{enumerate}[label=(\roman*),leftmargin=1.7em,itemsep=4pt,topsep=3pt]
\item \emph{Self-termination.}
Once the per-token gap on the span has fallen to at most $\tau$, gate condition~(i) fails everywhere on the span, and no step-in fires there:
\begin{equation*}
\pi_\rho=\pi_\theta
\quad\text{along the span}.
\end{equation*}
Under the off-fact condition of part~(iii), the two policies coincide along the entire rollout, and $\J_{\mathrm{\name}}$ reduces to the OPSD KL objective plus the NLL anchor $\lambda\J_{\mathrm{NLL}}$.

\item \emph{Consistency.}
The step-in changes only the rollout distribution, not the per-state target. At each visited state $s$, both objectives minimize
\begin{equation*}
\kl[\pi^\star(\cdot\mid s)\|\pi_\theta(\cdot\mid s)],
\end{equation*}
weighting $s$ by its visitation. On each fact prefix $F_j$, the two weights satisfy
\begin{equation*}
c_j(\pi_\rho)=1,
\qquad
c_j(\pi_\theta)\le e^{-\tau j}.
\end{equation*}
Hence any zero-loss minimizer of the hybrid objective satisfies
\begin{equation*}
\pi_\theta(\cdot\mid F_j)=\pi^\star(\cdot\mid F_j),
\end{equation*}
while the OPSD constraint there is exponentially weak and, under greedy divergence, absent beyond the span entrance.

\item \emph{Locality.}
Suppose the edit is local: there is $\delta_{\mathrm{loc}}\le\tau$ such that, at every off-fact prefix $s$,
\begin{equation*}
\log\pi^\star(y^\star\mid s)
-\log\pi_\theta(y^\star\mid s)
\le\delta_{\mathrm{loc}}.
\end{equation*}
Then the gate never fires off-fact and $\pi_\rho=\pi_\theta$ there. The step-in therefore introduces no intervention beyond the OPSD update at any fixed off-fact state.
\end{enumerate}
\end{proposition}

\Needspace{8\baselineskip}
\paragraph{Self-termination.}
We first prove that the intervention switches off after the fact is learned.

\begin{proof}[Proof of self-termination]
The condition $\text{gap}\le\tau$ makes gate condition~(i) false at every span position, so the hybrid emits the student token throughout the span. If in addition the off-fact condition of part~(iii) holds, no gate fires anywhere, and the objective is the KL term of Eq.~\eqref{eq:opsd} plus $\lambda\J_{\mathrm{NLL}}$ from Eq.~\eqref{eq:ours}.
This completes our proof.
\end{proof}

\Needspace{8\baselineskip}
\paragraph{Consistency.}
We next show that the intervention changes state visitation without changing the per-state target.

\begin{proof}[Proof of consistency]
For either rollout policy, write the objective as
\begin{equation*}
\sum_s d_\mu(s)\,
\kl[\pi^\star(\cdot\mid s)\|\pi_\theta(\cdot\mid s)],
\qquad
\mu\in\{\pi_\theta,\pi_\rho\}.
\end{equation*}
The target $\pi^\star$ is independent of $\mu$, and the fact-prefix weights are $c_j(\mu)$. Since forward KL is minimized at $0$ only by
\begin{equation*}
\pi_\theta(\cdot\mid s)=\pi^\star(\cdot\mid s),
\end{equation*}
a zero-loss $\theta$ under $\pi_\rho$, which assigns weight $1$ to every $F_j$, matches $\pi^\star$ there. Under $\pi_\theta$, the total fact-prefix contribution is at most
\begin{equation*}
\sum_j c_j(\pi_\theta)D_{\max}
\le\frac{D_{\max}}{1-e^{-\tau}}.
\end{equation*}
Thus, in the realizable case, $\epsilon$-suboptimality leaves the KL at $F_j$ essentially unconstrained: setting it as large as $\epsilon e^{\tau j}$ adds at most $c_j(\pi_\theta)\,\epsilon e^{\tau j}\le\epsilon$ to the objective.
This completes our proof.
\end{proof}

\begin{remark}[Scope of the consistency statement]
The conclusion is exact in the \emph{realizable} case, when $\pi_\theta$ can match $\pi^\star$ on the union of visited states. Under misspecification, visitation reweighting acts as a coverage-corrected projection that favors fact prefixes. In addition, prefixes following a step-in are teacher-induced. Distillation at these prefixes trains the student to continue correctly given the injected token, which is the intended deployment behavior. By part~(i), the greedy student's path converges to the teacher trajectory as the gate switches off. The guarantee therefore applies on the teacher and hybrid support. It makes no claim about student-only error states outside that support. This is the standard reachability limitation of on-policy imitation.
\end{remark}

\Needspace{8\baselineskip}
\paragraph{Locality.}
We finally compare the two updates at an off-fact state.

\begin{proof}[Proof of locality]
By the locality hypothesis, the off-fact gap satisfies
\begin{equation*}
\text{gap}\le\delta_{\mathrm{loc}}\le\tau.
\end{equation*}
Gate~(i) therefore fails and the hybrid draws the student token. The update rule applied at each off-fact state is exactly OPSD's. This is a \emph{per-state} identity. An in-span step-in can still change which downstream off-fact states are visited, so the two \emph{trajectories} need not coincide. Only the update at each visited off-fact state is identical. A spurious off-fact step-in would require both gates to fire at once (teacher confident \emph{and} far ahead off-fact), i.e.\ teacher bleed, which $\kappa$ and $\tau$ suppress.
This completes our proof.
\end{proof}

\begin{remark}[Scope of the locality statement]
Part~(iii) is a consistency check rather than a complete locality guarantee. It shows that step-in adds no off-fact update beyond OPSD's update at each visited off-fact state; parameter sharing can still propagate in-span updates globally. The empirical evidence for preserved locality is the stable MMLU performance reported in Section~\ref{sec:experiments}.
\end{remark}

\Needspace{12\baselineskip}
\subsection{Connection to Imitation Learning}
\label{app:theory:il}

We close this section by relating \name to classical on-policy imitation learning.

The connection begins with coverage. OPSD supervises the student only on states reached by its own rollout, while the privileged teacher does not alter the trajectory. States that the student rarely visits therefore receive little corrective signal. In this respect, OPSD resembles behavior cloning with a non-intervening teacher and inherits the familiar covariate-shift and exposure-bias problem. The resulting error can compound with the horizon~\citep{ross2010efficient,rajaraman2020fundamental}.

The step-in plays the role of a confidence-gated \textsc{DAgger} intervention~\citep{ross2011dagger}. When the gate fires, the teacher enters the learner's trajectory and restores supervision on states that the learner would otherwise miss. Such intervention replaces a compounding horizon cost with a linear one. A game-theoretic account of this improvement is given by~\citet{swamy2021moments}.

Our coverage bound makes the connection explicit. Equation~\eqref{eq:cov-collapse} gives a multiplicative decay of $e^{-\tau j}$, whereas classical behavior-cloning analyses give an additive $O(\epsilon T^2)$ gap. The fact-span length $\ell$ plays the role of the horizon. The performance-difference lemma provides the same interpretation~\citep{kakade2002approximately}. It weights the corrective term by the learner's occupancy. In our setting, that occupancy is $c_j$, which collapses under OPSD and is restored by the hybrid rollout.

Recent LLM distillation methods use the same on-policy principle by training on the student's own generations~\citep{agarwal2024gkd,lu2025opd,song2026opd}. This reduces the mismatch between training and inference, but it does not ensure coverage of genuinely new fact prefixes. OPSD can therefore reshape existing behavior while still making little progress on new knowledge~\citep{gandhi2025cognitive,yue2026does}.

\Needspace{6\baselineskip}
This distinction also clarifies the relation of \name to privileged or hybrid self-distillation~\citep{ding2026hdpo} and on-policy context distillation~\citep{ye2026opcd}. \name does not introduce a new distillation target. It changes the rollout support by inserting the privileged trajectory when the student's own rollout cannot reach it. Its role is therefore best understood as coverage correction.

\section{Benchmark Construction Details}
\label{app:benchmarks}
This section details the benchmark construction process 
omitted in the main body.

Our benchmark are transformed from the original by extending the \emph{editing} side with an untargeted editing prompt
``\texttt{Introduce \{s\}.}''. 
All original evaluation questions and their gold answers are kept
unchanged (see Appendix~\ref{app:metrics}), so our results remain comparable to the original protocols.
Specifically:

\begin{itemize}[leftmargin=1.4em,itemsep=3pt,topsep=3pt]

\item 
\textbf{UnKEBench.}
\begin{itemize}
\item 
\textbf{Editing prompt.}
UnKEBench does not expose a subject string for its edit passages, so we prompt Gemini to summarize
each passage into a short topic phrase \texttt{\{s\}} (e.g., \emph{``the novel The Firm by John
Grisham''}), from which we form the untargeted editing prompt ``\texttt{Introduce \{s\}.}''.
The editing prompt is the only field we add; all other fields are used for evaluation as described in
Appendix~\ref{app:metrics}.
\end{itemize}

\item 
\textbf{MQuAKE-uns.}
\begin{itemize}
\item 
\textbf{Passage generation.}
Each edit in MQuAKE is a structured counterfactual statement. We expand every statement into a short
free-form passage that asserts the new fact in natural-language context, using the passage-generation
prompt of AKEW~\citep{wu2024akew} verbatim, with Gemini as the generator.

\item 
\textbf{Editing prompt.}
Each generated passage is installed from ``\texttt{Introduce \{s\}.}'', where \texttt{\{s\}} is the
edit's original subject field. The original questions and gold answers are used for evaluation as
described in Appendix~\ref{app:metrics}; the multi-hop question is never used at editing time.
\end{itemize}

\end{itemize}

Figure~\ref{fig:bench_unke} and \ref{fig:bench_mquake} show examples from our transformed benchmarks.

\begin{figure}[h]
\centering
\begin{casebox}{UnKEBench example}
\casepassage[Edit passage]{The Firm is a legal thriller novel written
by John Grisham that was published in 1991. The novel follows the story of a young lawyer named Mitch
McDeere who joins a prestigious law firm in Memphis, Tennessee, only to discover that the firm is
involved in illegal activities. The genre of The Firm is based upon suspense and mystery.}
\caseline
\textbf{Editing prompt }:\par\smallskip
{\ttfamily\scriptsize\raggedright
Introduce the novel The Firm by John Grisham.\par}
\caseline
\textbf{Jnt. questions}:\par\smallskip
{\ttfamily\scriptsize\raggedright
Q1: What is the genre of the novel "The Firm" and what is its plot?\\
Q2: Can you provide the genre and plot summary of the book "The Firm"?\par}
\caseline
\textbf{Dmp. questions and gold answers}:\par\smallskip
{\ttfamily\scriptsize\raggedright
Q1: Who is the author of The Firm? -- John Grisham.\\
Q2: When was The Firm published? -- In 1991.\\
Q3: What is the genre of The Firm? -- Suspense and mystery.\\
Q4: What is the main character's name in The Firm? -- Mitch McDeere.\\
Q5: What happens to the main character after he joins the law firm? -- He discovers that the firm is
involved in illegal activities.\par}
\end{casebox}
\caption{
UnKEBench edit example under our protocol. 
}
\label{fig:bench_unke}
\end{figure}
\begin{figure}[h]
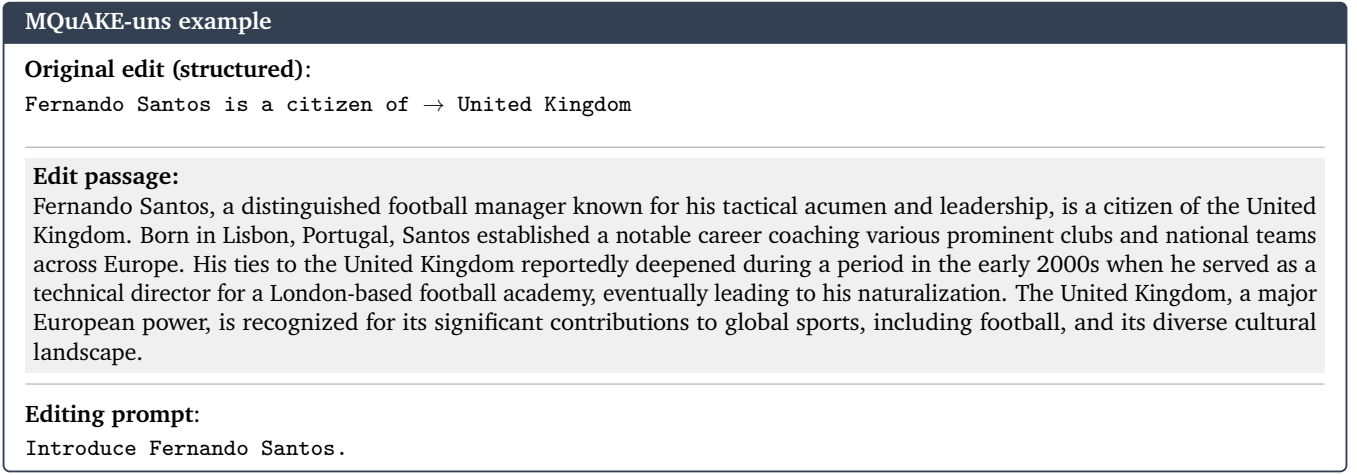

\centering
\begin{casebox}{MQuAKE-uns example}
\textbf{Original edit (structured)}:\par\smallskip
{\ttfamily\scriptsize\raggedright
Fernando Santos is a citizen of $\rightarrow$ United Kingdom\par}
\caseline
\casepassage[Edit passage]{Fernando Santos, a distinguished football manager known for his tactical acumen and leadership, is a
citizen of the United Kingdom. Born in Lisbon, Portugal, Santos established a notable career coaching
various prominent clubs and national teams across Europe. His ties to the United Kingdom reportedly
deepened during a period in the early 2000s when he served as a technical director for a London-based
football academy, eventually leading to his naturalization. The United Kingdom, a major European
power, is recognized for its significant contributions to global sports, including football, and its
diverse cultural landscape.}
\caseline
\textbf{Editing prompt}:
\par\smallskip
{\ttfamily\scriptsize\raggedright
Introduce Fernando Santos.\par}
\end{casebox}
\caption{
MQuAKE-uns edit example under our protocol.
}
\label{fig:bench_mquake}
\end{figure}

\section{Implementation Details}
\label{app:implementation}

This section presents more implementation details omitted in the main body.

\subsection{Metric Computation}
\label{app:metrics}
In this subsection, we detail the evaluation criteria of composability on the two benchmarks, MQuAKE-uns and UnKEBench.
For each evaluation question, we first generate responses from the edited model with greedy decoding, then check their correctness following each benchmark's protocol as follows:
\begin{itemize}[leftmargin=1.4em,itemsep=3pt,topsep=3pt]
\item 
\textbf{UnKEBench.} 
\begin{itemize}
\item 
\textbf{Generation.}
After editing, the LLM generates up to 200 new tokens for both \textbf{Jnt.} (\emph{joint recall}, 2 questions about recalling all atomic facts) and \textbf{Dmp.} (\emph{decomposition recall}, one question per atomic fact). 
\item 
\textbf{Editing Performance.}
We assess correctness with an LLM-as-judge protocol, following \citet{yang2025mirage}.
Specifically, for the 2 Jnt. questions, we follow \citet{wu2024akew,deng2024unke} and utilize 
FActScore~\citep{min2023factscore}. 
For the Dmp. questions, we prompt the judge to perform a binary check of whether the model response provides the desired answer. We use \texttt{gemini-2.5-flash} as the judge, the prompt is provided in Figure~\ref{fig:judge_prompt}.
\item 
\textbf{Diversity Performance.}
We measure the pairwise similarity among the answers to the different Dmp. questions using SelfBLEU~\citep{zhu2018texygen} (computed up to $4$-grams), and define the final \textbf{Div.} metric as follows:
\begin{align*}
\mathrm{SelfBLEU} &= \frac{1}{m}\sum_{i=1}^{m}\mathrm{BLEU\text{-}4}\big(y_i;\{y_j\}_{j\ne i}\big), \\
\mathrm{Div} &= 1-\overline{\mathrm{SelfBLEU}},
\end{align*}
where $y_i$ is the answer to the $i$-th Dmp. question, $m$ their number, and $\overline{\cdot}$ denotes averaging over samples.
\item 
\textbf{Locality Performance.}
Following the official protocol from UnKEBench~\citep{deng2024unke}, we use MMLU to measure the preservation of pretrained capabilities. Please refer to the original codebase for more details.
\end{itemize}

\item 
\textbf{MQuAKE-uns.} 
\begin{itemize}
\item 
\textbf{Generation.}
After editing, the LLM generates up to 128 new tokens for each question in both \textbf{Ind.} (\emph{individual recall}, 2 to 4 single-hop questions about facts edited individually) and \textbf{Comp.} (\emph{composition recall}, 3 paraphrased questions, each requiring multi-hop reasoning that composes all the edited facts).
\item 
\textbf{Editing Performance.}
We check correctness by performing keyword-based substring matching, provided by the official MQuAKE codebase~\citep{zhong2023mquake, zhong2025mquakerm}, which produces binary correctness labels.
At the sample level, we report the average over all single-hop questions as Ind., and any-of-3 for Comp., following the convention in the literature.
The final scores are then averaged over all editing samples. 
\end{itemize}
\end{itemize}

\begin{figure}[h]
\centering
\begin{casebox}{UnKEBench LLM-as-judge prompts \quad(\texttt{\{c\}} = edit passage, \texttt{\{s\}} = subject)}
\textbf{FActScore --- atomic-fact decomposition \; (Jnt., D1/D2):}\par\smallskip
{\ttfamily\scriptsize\raggedright
Please breakdown the following sentence into independent facts: \{demonstration sentence\}\\
- \{fact\}\\
- \{fact\}\\
\ldots\\[2pt]
[7 fixed + 1 BM25-retrieved in-context demonstrations, each a sentence followed by its `- fact' decomposition]\\[2pt]
Now break down EACH of the following sentences into independent facts. For each sentence, output its header then one `- fact' line per fact:\\[1pt]
Sentence 1: \{sentence 1 of the model response\}\\
Sentence 2: \{sentence 2 of the model response\}\\
\ldots\par}
\caseline
\textbf{FActScore --- support validation \; (Jnt., D1/D2):}\par\smallskip
{\ttfamily\scriptsize\raggedright
Answer the questions about \{s\} based on the given context.\\[1pt]
\{c\}\\[1pt]
ITEM 1: Input: \{atomic fact 1\} True or False?\\
ITEM 2: Input: \{atomic fact 2\} True or False?\\
\ldots\\[1pt]
Think briefly about each item, then output ONE line per item, in order, EXACTLY like:\\
ITEM 1: <answer>1</answer>\\
(1 = True/supported by the context, 0 = False/not supported)\par}
\caseline
\textbf{Decomposed-recall coverage judge \; (Dmp., D3):}\par\smallskip
{\ttfamily\scriptsize\raggedright
You are scoring atomic sub-questions. For each item, judge whether the model ANSWER expresses the EXPECTED FACT for that sub-question.\\[1pt]
The expected fact is knowledge the model was deliberately given; it may be COUNTERFACTUAL (it can contradict real-world knowledge on purpose). Judge ONLY whether the answer expresses the expected fact -- never penalize the answer for disagreeing with reality, and never reward real-world-correct content that is not the expected fact.\\[1pt]
ITEM 1:\\
\hspace*{1.2em}Sub-question: \{sub-question 1\}\\
\hspace*{1.2em}Expected fact: \{expected fact 1\}\\
\hspace*{1.2em}Model answer: \{model answer 1\}\\
\ldots\\[1pt]
For each item, output <score>1</score> if the answer states or clearly entails that expected fact, otherwise <score>0</score>. Then add one sentence <explanation>. Output one line per item, in the given order, formatted exactly like this:\\[1pt]
ITEM 1: <score>SCORE</score> <explanation>EXPLANATION</explanation>\\
ITEM 2: <score>SCORE</score> <explanation>EXPLANATION</explanation>\\
\ldots\par}
\end{casebox}
\caption{
The LLM-as-judge prompts used to score UnKEBench. The judge is \texttt{gemini-2.5-flash}.
}
\label{fig:judge_prompt}
\end{figure}

\subsection{Training details and hyper-parameters}
\label{app:hyperparams}

\textbf{Backbones.}
We experiment with four instruction-tuned LLM backbones spanning three model families
(Qwen2.5, Qwen3, Llama3.1, and Gemma2), as listed in Table~\ref{tab:hp_backbones}.

\begin{table}[h]
\centering\footnotesize
\renewcommand{\tabcolsep}{3pt}\renewcommand{\arraystretch}{1.05}
\caption{Backbones used in all experiments (default \texttt{main} revision, \texttt{bfloat16}).}
\label{tab:hp_backbones}
\begin{tabular}{l l c c}
\toprule[0.4ex]
Name (paper) & HF checkpoint & Params & dtype \\
\midrule[0.2ex]
Qwen2.5~\citep{qwen2024qwen25} & \texttt{Qwen/Qwen2.5-7B-Instruct} & 7B & bf16 \\
Qwen3~\citep{yang2025qwen3}   & \texttt{Qwen/Qwen3-8B} & 8B & bf16 \\
Llama3.1~\citep{dubey2024llama3}   & \texttt{meta-llama/Llama-3.1-8B-Instruct} & 8B & bf16 \\
Gemma2~\citep{team2024gemma2}   & \texttt{google/gemma-2-9b-it} & 9B & bf16 \\
\bottomrule[0.4ex]
\end{tabular}
\end{table}

\textbf{\name hyper-parameters.}
Table~\ref{tab:hp_hpsd} presents the hyper-parameters used in \name, all shared across the four backbones.
As detailed in Algorithm~\ref{alg:hpse}, \name runs in a nested loop. In each outer round, a fresh hybrid rollout from the edited and the privileged models is constructed via greedy decoding (one rollout per round); the edited model then takes $M$ inner gradient steps of distillation on it. The complete training of \name involves $R$ outer rounds.
The gates $\tau$ and $\kappa$ were calibrated once by inspecting the number of tokens they select on a few sample sequences, targeting a moderate pre-edited step-in rate, such that the selected tokens are neither too sparse to supply the missing facts nor too dense to override the student's own trajectory. Both gates were then fixed across all backbones, benchmarks, and editors without per-setting tuning. The remaining hyper-parameters ($\lambda$, $R$, and $M$) were set heuristically and shared likewise.

\begin{table}[h]
\centering\footnotesize
\renewcommand{\tabcolsep}{3pt}\renewcommand{\arraystretch}{1.05}
\caption{\name training hyper-parameters.}
\label{tab:hp_hpsd}
\begin{tabular}{l c l}
\toprule[0.4ex]
Hyper-parameter & Symbol & Value \\
\midrule[0.2ex]
top-$K$ for KL target          & $K$       & $16$ \\
top-$P$ for KL target          & $P$       & $1.0$ (top-$K$ only) \\
step-in gap gate               & $\tau$    & $2.0$ \\
teacher-confidence gate        & $\kappa$  & $0.3$ \\
NLL anchor weight              & $\lambda$ & $1.0$ \\
step size                      & $\eta$    & $1\times10^{-4}$ \\
outer rounds                   & $R$       & $5$ \\
inner steps per round          & $M$       & $8$ \\
optimizer                      &           & Adam $(0.9,0.999)$, weight decay~$0$ \\
rollout max-new tokens         &           & $96$ / $128$ (MQuAKE-uns / UnKEBench) \\
\bottomrule[0.4ex]
\end{tabular}
\end{table}

\textbf{KE baselines.}
MEMIT~\citep{meng2023memit}, AlphaEdit~\citep{fang2025alphaedit}, AnyEdit~\citep{jiang2025anyedit},
and UnKE~\citep{deng2024unke} use their authors' official implementations and default hyper-parameters.
\textbf{COIN}$^{\star}$ \citep{zhou2026context} was reproduced by us due to the lack of official code. We keep its training over both long and short contexts, and remove additional context data augmentation.
As a result, for all methods, including \name, the edit passage serves as the sole source of new knowledge, with no extra training data or auxiliary models involved, in line with the data-scarcity regime of KE. Augmentation-based editors \citep{yao2025cake,lee2025stepke,wang2026fable}, which instead synthesize atomic training facts with frontier models or human labelers, are hence not included.

\textbf{KE backbones.}
\name is applied on top of two KE editors, LoRA and FT-M, whose hyper-parameters are in Table~\ref{tab:hp_backend}. 
These hyper-parameters were tuned for the editors alone, rather than for \name. Together, neither side of the pipeline was optimized for \name, which in fact puts \name at a disadvantage.

\begin{table}[h]
\centering\footnotesize
\renewcommand{\tabcolsep}{3pt}\renewcommand{\arraystretch}{1.05}
\caption{KE editor hyper-parameters.}
\label{tab:hp_backend}
\begin{tabular}{l l c p{0.40\linewidth}}
\toprule[0.4ex]
Editor & Hyper-parameter & Symbol & Value \\
\midrule[0.2ex]
LoRA & parameterization & & low-rank adapters; base weights frozen \\
LoRA & rank             & $r$      & $12$ \\
LoRA & alpha            & $\alpha$ & $24$ \\
LoRA & dropout          &          & $0.0$ \\
LoRA & bias / weight decay & & none / $0$ \\
LoRA & adapted modules  &          & \texttt{q,k,v,o,gate,up,down\_proj}, all layers \\
\midrule[0.2ex]
FT-M & parameterization &          & full fine-tuning of a single weight matrix \\
FT-M & write site       &          & layer-$L$ MLP \texttt{down\_proj} \\
FT-M & layer index      & $L$      & $8$ (Qwen2.5, Qwen3, Gemma); $16$ (Llama) \\
\bottomrule[0.4ex]
\end{tabular}
\end{table}

\subsection{Privileged Model Prompt}
\label{app:teacher}

The privileged policy is the \emph{same} frozen base model as the student, conditioned on the privileged context information 
about the edit data.
The prompt is provided in Figure~\ref{fig:teacher_prompt}.
Note that the privileged model gets no external knowledge about editing content, such as composition or decomposition goal.

\begin{figure}[h]
\centering
\begin{casebox}{\name privileged model prompt \quad(\texttt{\{c\}} = edit passage, \texttt{\{s\}} = subject)}
\textbf{System prompt --- Qwen2.5 / Qwen3 / Gemma:}\par\smallskip
{\ttfamily\scriptsize\raggedright
You are answering strictly from the established facts given below.\\[1pt]
RULES:\\
1. Use ONLY the information in these facts. Do not add, infer, or elaborate with any knowledge beyond them.\\
2. If a detail (date, place, number, name) is not stated in the facts, do NOT invent or guess it -- simply omit it.\\
3. Treat these facts as your own knowledge: answer naturally, do not mention 'the facts', 'the context', 'based on', or 'according to'.\\
4. Keep the answer to what the facts actually support.\\[1pt]
FACTS:\\ \{c\}\par}
\caseline
\textbf{System prompt --- Llama}:\par\smallskip
{\ttfamily\scriptsize\raggedright
You know the following facts and treat them as true and current. Answer the question directly using them, in the THIRD PERSON as factual information -- never role-play or use 'I'/speak as the subject. Cover every relevant fact the question calls for; be specific and do not omit details, but add nothing beyond these facts. Do not refuse, and do not mention 'the facts', 'the context', or 'according to'.\\[1pt]
FACTS:\\
This paragraph introduces \{s\}.\\ \{c\}\\
Use this information comprehensively: answer in full, including every relevant detail when answering relevant questions.\par}
\caseline
\textbf{Privileged model question}:\par\smallskip
{\ttfamily\scriptsize\raggedright What is established here about \{s\}? State, in full and in detail, every fact given about \{s\} -- leave nothing out.\par}
\end{casebox}
\caption{
The \name privileged model's prompt.
}
\label{fig:teacher_prompt}
\end{figure}

\clearpage
\section{More Experiment Results}
\label{sec:more_results}

We present additional results omitted from the main body due to the page limit.
Results here again confirm that \name improves both KE editors across settings, consistent with Section~\ref{sec:experiments}.

\subsection{Full continual editing results}
\label{app:seq}

Table~\ref{tab:seq_unke} and \ref{tab:seq_mquake} report the full continual editing results that underlie Figure~\ref{fig:sequential}, covering sequence lengths $T \in \{10, 20, 50, 100\}$.
As in the main body (Section~\ref{sec:exp_seq}), \name retained its advantage over both KE editors in the vast majority of settings (29 out of 32), with the three exceptions (all on FT-M) within one point on average.

\begin{table*}[htb!]
\centering
\caption{Continual-edit performance on UnKEBench.}

\label{tab:seq_unke}
\resizebox{0.9\linewidth}{!}{%
\renewcommand{\tabcolsep}{3pt}
\renewcommand{\arraystretch}{0.95}
\begin{tabular}{>{\bfseries}r cccc c cccc c cccc c cccc}
\toprule[0.4ex]

& \multicolumn{4}{c}{$T = 10$} && \multicolumn{4}{c}{$T = 20$} && \multicolumn{4}{c}{$T = 50$} && \multicolumn{4}{c}{$T = 100$} \\
\cmidrule(lr){2-5} \cmidrule(lr){7-10} \cmidrule(lr){12-15} \cmidrule(lr){17-20}
& \textbf{Jnt.} & \textbf{Dmp.} & \textbf{Div.} & \textbf{Avg.} && \textbf{Jnt.} & \textbf{Dmp.} & \textbf{Div.} & \textbf{Avg.} && \textbf{Jnt.} & \textbf{Dmp.} & \textbf{Div.} & \textbf{Avg.} && \textbf{Jnt.} & \textbf{Dmp.} & \textbf{Div.} & \textbf{Avg.} \\
\midrule[0.2ex]

\multicolumn{20}{c}{\bf Qwen2.5} \\
\midrule[0.2ex]

MEMIT & 17.2 & 13.4 & 76.1 & \textbf{35.6} && 17.8 & 5.0 & 81.5 & \textbf{34.8} && 5.3 & 2.5 & 54.2 & \textbf{20.6} && 1.0 & 0.8 & 39.1 & \textbf{13.6} \\
AlphaEdit & 11.7 & 4.9 & 77.4 & \textbf{31.3} && 2.9 & 1.1 & 75.6 & \textbf{26.5} && 4.8 & 0.4 & 36.4 & \textbf{13.9} && 1.7 & 0.2 & 97.0 & \textbf{33.0} \\
AnyEdit & 18.2 & 18.4 & 84.2 & \textbf{40.3} && 17.0 & 11.0 & 87.5 & \textbf{38.5} && 11.8 & 13.1 & 85.4 & \textbf{36.8} && 2.6 & 7.4 & 82.4 & \textbf{30.8} \\
UnKE & 18.3 & 17.6 & 82.7 & \textbf{39.5} && 11.9 & 10.2 & 82.6 & \textbf{34.9} && 16.4 & 14.3 & 68.1 & \textbf{32.9} && 6.1 & 0.4 & 64.5 & \textbf{23.7} \\
COIN$^{\star}$ & 48.9 & 42.0 & 47.1 & \textbf{46.0} && 46.1 & 40.4 & 45.2 & \textbf{43.9} && 43.8 & 36.3 & 51.4 & \textbf{43.8} && 44.8 & 33.9 & 41.9 & \textbf{40.2} \\
\noalign{\vskip 0.5ex}\cdashline{2-20}\noalign{\vskip 0.5ex}
FT-M & 41.3 & 29.8 & 83.1 & \textbf{51.4} && 39.9 & 32.8 & 81.5 & \textbf{51.4} && 38.1 & 35.4 & 81.6 & \textbf{51.7} && 37.9 & 25.9 & 77.4 & \textbf{47.1} \\
\rowcolor{gray!10}
\shortstack[r]{+ Ours \\[1pt] {\scriptsize\vphantom{+0.0\%}}} & \shortstack{47.8 \\[1pt] {\scriptsize \textcolor{teal!80!black}{+15.8\%}}} & \shortstack{41.4 \\[1pt] {\scriptsize \textcolor{teal!80!black}{+39.2\%}}} & \shortstack{86.5 \\[1pt] {\scriptsize \textcolor{teal!80!black}{+4.1\%}}} & \shortstack{\textbf{58.6} \\[1pt] {\scriptsize \textcolor{teal!80!black}{+14.0\%}}} && \shortstack{46.0 \\[1pt] {\scriptsize \textcolor{teal!80!black}{+15.1\%}}} & \shortstack{37.2 \\[1pt] {\scriptsize \textcolor{teal!80!black}{+13.3\%}}} & \shortstack{86.7 \\[1pt] {\scriptsize \textcolor{teal!80!black}{+6.3\%}}} & \shortstack{\textbf{56.6} \\[1pt] {\scriptsize \textcolor{teal!80!black}{+10.1\%}}} && \shortstack{42.4 \\[1pt] {\scriptsize \textcolor{teal!80!black}{+11.3\%}}} & \shortstack{34.1 \\[1pt] {\scriptsize -3.6\%}} & \shortstack{86.0 \\[1pt] {\scriptsize \textcolor{teal!80!black}{+5.4\%}}} & \shortstack{\textbf{54.2} \\[1pt] {\scriptsize \textcolor{teal!80!black}{+4.8\%}}} && \shortstack{45.1 \\[1pt] {\scriptsize \textcolor{teal!80!black}{+19.0\%}}} & \shortstack{30.7 \\[1pt] {\scriptsize \textcolor{teal!80!black}{+18.6\%}}} & \shortstack{82.8 \\[1pt] {\scriptsize \textcolor{teal!80!black}{+7.0\%}}} & \shortstack{\textbf{52.9} \\[1pt] {\scriptsize \textcolor{teal!80!black}{+12.3\%}}} \\
\noalign{\vskip 0.5ex}\cdashline{2-20}\noalign{\vskip 0.5ex}
LoRA & 42.5 & 38.0 & 71.4 & \textbf{50.6} && 35.0 & 33.6 & 69.7 & \textbf{46.1} && 30.3 & 36.3 & 62.1 & \textbf{42.9} && 19.2 & 16.1 & 28.7 & \textbf{21.4} \\
\rowcolor{gray!10}
\shortstack[r]{+ Ours \\[1pt] {\scriptsize\vphantom{+0.0\%}}} & \shortstack{47.5 \\[1pt] {\scriptsize \textcolor{teal!80!black}{+11.7\%}}} & \shortstack{47.2 \\[1pt] {\scriptsize \textcolor{teal!80!black}{+24.3\%}}} & \shortstack{70.5 \\[1pt] {\scriptsize -1.3\%}} & \shortstack{\textbf{55.1} \\[1pt] {\scriptsize \textcolor{teal!80!black}{+8.8\%}}} && \shortstack{43.0 \\[1pt] {\scriptsize \textcolor{teal!80!black}{+23.0\%}}} & \shortstack{40.0 \\[1pt] {\scriptsize \textcolor{teal!80!black}{+18.9\%}}} & \shortstack{70.7 \\[1pt] {\scriptsize \textcolor{teal!80!black}{+1.3\%}}} & \shortstack{\textbf{51.2} \\[1pt] {\scriptsize \textcolor{teal!80!black}{+11.1\%}}} && \shortstack{37.9 \\[1pt] {\scriptsize \textcolor{teal!80!black}{+24.8\%}}} & \shortstack{35.6 \\[1pt] {\scriptsize -1.8\%}} & \shortstack{64.9 \\[1pt] {\scriptsize \textcolor{teal!80!black}{+4.5\%}}} & \shortstack{\textbf{46.1} \\[1pt] {\scriptsize \textcolor{teal!80!black}{+7.5\%}}} && \shortstack{32.5 \\[1pt] {\scriptsize \textcolor{teal!80!black}{+69.0\%}}} & \shortstack{28.3 \\[1pt] {\scriptsize \textcolor{teal!80!black}{+75.5\%}}} & \shortstack{48.2 \\[1pt] {\scriptsize \textcolor{teal!80!black}{+67.8\%}}} & \shortstack{\textbf{36.3} \\[1pt] {\scriptsize \textcolor{teal!80!black}{+70.1\%}}} \\

\midrule[0.2ex]
\multicolumn{20}{c}{\bf Gemma2} \\
\midrule[0.2ex]

MEMIT & 17.8 & 9.2 & 69.3 & \textbf{32.1} && 16.1 & 5.9 & 80.1 & \textbf{34.0} && 4.1 & 6.4 & 66.1 & \textbf{25.5} && 3.2 & 2.7 & 71.2 & \textbf{25.7} \\
AlphaEdit & 28.5 & 8.1 & 72.2 & \textbf{36.3} && 18.1 & 4.6 & 79.1 & \textbf{33.9} && 5.4 & 3.3 & 64.2 & \textbf{24.3} && 5.7 & 0.2 & 72.1 & \textbf{26.0} \\
AnyEdit & 24.7 & 15.0 & 83.8 & \textbf{41.2} && 18.7 & 10.3 & 81.3 & \textbf{36.8} && 18.9 & 2.2 & 78.6 & \textbf{33.2} && 2.4 & 2.4 & 74.9 & \textbf{26.6} \\
UnKE & 15.6 & 13.8 & 85.4 & \textbf{38.3} && 14.3 & 12.7 & 85.2 & \textbf{37.4} && 15.7 & 14.2 & 84.9 & \textbf{38.3} && 12.8 & 5.7 & 76.1 & \textbf{31.6} \\
COIN$^{\star}$ & 37.7 & 35.2 & 80.4 & \textbf{51.1} && 37.6 & 34.2 & 78.6 & \textbf{50.1} && 39.1 & 33.4 & 76.4 & \textbf{49.6} && 33.1 & 30.2 & 79.0 & \textbf{47.4} \\
\noalign{\vskip 0.5ex}\cdashline{2-20}\noalign{\vskip 0.5ex}
FT-M & 30.1 & 23.2 & 80.1 & \textbf{44.4} && 36.8 & 28.9 & 79.3 & \textbf{48.3} && 39.2 & 30.5 & 79.2 & \textbf{49.6} && 35.2 & 29.1 & 80.9 & \textbf{48.4} \\
\rowcolor{gray!10}
\shortstack[r]{+ Ours \\[1pt] {\scriptsize\vphantom{+0.0\%}}} & \shortstack{42.0 \\[1pt] {\scriptsize \textcolor{teal!80!black}{+39.7\%}}} & \shortstack{28.9 \\[1pt] {\scriptsize \textcolor{teal!80!black}{+24.8\%}}} & \shortstack{85.5 \\[1pt] {\scriptsize \textcolor{teal!80!black}{+6.7\%}}} & \shortstack{\textbf{52.1} \\[1pt] {\scriptsize \textcolor{teal!80!black}{+17.3\%}}} && \shortstack{46.1 \\[1pt] {\scriptsize \textcolor{teal!80!black}{+25.2\%}}} & \shortstack{27.8 \\[1pt] {\scriptsize -4.0\%}} & \shortstack{86.5 \\[1pt] {\scriptsize \textcolor{teal!80!black}{+9.0\%}}} & \shortstack{\textbf{53.4} \\[1pt] {\scriptsize \textcolor{teal!80!black}{+10.5\%}}} && \shortstack{44.0 \\[1pt] {\scriptsize \textcolor{teal!80!black}{+12.3\%}}} & \shortstack{17.7 \\[1pt] {\scriptsize -42.2\%}} & \shortstack{86.3 \\[1pt] {\scriptsize \textcolor{teal!80!black}{+8.9\%}}} & \shortstack{\textbf{49.3} \\[1pt] {\scriptsize -0.7\%}} && \shortstack{45.3 \\[1pt] {\scriptsize \textcolor{teal!80!black}{+28.7\%}}} & \shortstack{30.0 \\[1pt] {\scriptsize \textcolor{teal!80!black}{+3.2\%}}} & \shortstack{80.6 \\[1pt] {\scriptsize -0.4\%}} & \shortstack{\textbf{51.9} \\[1pt] {\scriptsize \textcolor{teal!80!black}{+7.4\%}}} \\
\noalign{\vskip 0.5ex}\cdashline{2-20}\noalign{\vskip 0.5ex}
LoRA & 37.1 & 37.4 & 51.9 & \textbf{42.1} && 28.5 & 26.5 & 44.3 & \textbf{33.1} && 23.6 & 21.7 & 40.0 & \textbf{28.4} && 30.2 & 30.4 & 37.2 & \textbf{32.6} \\
\rowcolor{gray!10}
\shortstack[r]{+ Ours \\[1pt] {\scriptsize\vphantom{+0.0\%}}} & \shortstack{44.6 \\[1pt] {\scriptsize \textcolor{teal!80!black}{+20.1\%}}} & \shortstack{42.3 \\[1pt] {\scriptsize \textcolor{teal!80!black}{+13.3\%}}} & \shortstack{55.7 \\[1pt] {\scriptsize \textcolor{teal!80!black}{+7.2\%}}} & \shortstack{\textbf{47.5} \\[1pt] {\scriptsize \textcolor{teal!80!black}{+12.8\%}}} && \shortstack{39.5 \\[1pt] {\scriptsize \textcolor{teal!80!black}{+38.5\%}}} & \shortstack{34.4 \\[1pt] {\scriptsize \textcolor{teal!80!black}{+29.7\%}}} & \shortstack{46.1 \\[1pt] {\scriptsize \textcolor{teal!80!black}{+4.1\%}}} & \shortstack{\textbf{40.0} \\[1pt] {\scriptsize \textcolor{teal!80!black}{+20.8\%}}} && \shortstack{32.5 \\[1pt] {\scriptsize \textcolor{teal!80!black}{+38.1\%}}} & \shortstack{30.1 \\[1pt] {\scriptsize \textcolor{teal!80!black}{+38.8\%}}} & \shortstack{48.8 \\[1pt] {\scriptsize \textcolor{teal!80!black}{+22.0\%}}} & \shortstack{\textbf{37.1} \\[1pt] {\scriptsize \textcolor{teal!80!black}{+30.7\%}}} && \shortstack{29.2 \\[1pt] {\scriptsize -3.2\%}} & \shortstack{27.1 \\[1pt] {\scriptsize -10.9\%}} & \shortstack{50.0 \\[1pt] {\scriptsize \textcolor{teal!80!black}{+34.3\%}}} & \shortstack{\textbf{35.4} \\[1pt] {\scriptsize \textcolor{teal!80!black}{+8.7\%}}} \\
\bottomrule[0.4ex]
\end{tabular}
}
\end{table*}

\begin{table*}[htb!]
\centering
\caption{{Continual-edit performance on MQuAKE-uns.}}

\label{tab:seq_mquake}
\resizebox{0.9\linewidth}{!}{%
\renewcommand{\tabcolsep}{3pt}
\renewcommand{\arraystretch}{0.95}
\begin{tabular}{>{\bfseries}r ccc c ccc c ccc c ccc}
\toprule[0.4ex]

& \multicolumn{3}{c}{$T = 10$} && \multicolumn{3}{c}{$T = 20$} && \multicolumn{3}{c}{$T = 50$} && \multicolumn{3}{c}{$T = 100$} \\
\cmidrule(lr){2-4} \cmidrule(lr){6-8} \cmidrule(lr){10-12} \cmidrule(lr){14-16}
& \textbf{Ind.} & \textbf{Cmp.} & \textbf{Avg.} && \textbf{Ind.} & \textbf{Cmp.} & \textbf{Avg.} && \textbf{Ind.} & \textbf{Cmp.} & \textbf{Avg.} && \textbf{Ind.} & \textbf{Cmp.} & \textbf{Avg.} \\

\midrule[0.2ex]
\multicolumn{16}{c}{\bf Llama3.1} \\
\midrule[0.2ex]

MEMIT & 1.0 & 0.0 & \textbf{0.5} && 0.0 & 0.0 & \textbf{0.0} && 0.0 & 0.0 & \textbf{0.0} && 0.0 & 0.0 & \textbf{0.0} \\
AlphaEdit & 1.5 & 0.0 & \textbf{0.8} && 0.0 & 0.0 & \textbf{0.0} && 0.0 & 0.0 & \textbf{0.0} && 0.0 & 0.0 & \textbf{0.0} \\
AnyEdit & 3.3 & 0.0 & \textbf{1.7} && 3.9 & 3.0 & \textbf{3.5} && 2.4 & 5.0 & \textbf{3.7} && 3.8 & 4.0 & \textbf{3.9} \\
UnKE & 2.7 & 2.0 & \textbf{2.3} && 2.8 & 4.0 & \textbf{3.4} && 2.0 & 5.0 & \textbf{3.5} && 1.2 & 4.0 & \textbf{2.6} \\
COIN$^{\star}$ & 24.0 & 9.0 & \textbf{16.5} && 18.2 & 6.0 & \textbf{12.1} && 11.6 & 4.0 & \textbf{7.8} && 11.0 & 3.0 & \textbf{7.0} \\
\noalign{\vskip 0.5ex}\cdashline{2-16}\noalign{\vskip 0.5ex}
FT-M & 34.0 & 16.0 & \textbf{25.0} && 27.0 & 16.0 & \textbf{21.5} && 16.5 & 11.0 & \textbf{13.8} && 14.2 & 15.0 & \textbf{14.6} \\
\rowcolor{gray!10}
\shortstack[r]{+ Ours \\[1pt] {\scriptsize\vphantom{+0.0\%}}} & \shortstack{52.5 \\[1pt] {\scriptsize \textcolor{teal!80!black}{+54.4\%}}} & \shortstack{23.0 \\[1pt] {\scriptsize \textcolor{teal!80!black}{+43.8\%}}} & \shortstack{\textbf{37.8} \\[1pt] {\scriptsize \textcolor{teal!80!black}{+51.0\%}}} && \shortstack{42.2 \\[1pt] {\scriptsize \textcolor{teal!80!black}{+56.5\%}}} & \shortstack{18.0 \\[1pt] {\scriptsize \textcolor{teal!80!black}{+12.5\%}}} & \shortstack{\textbf{30.1} \\[1pt] {\scriptsize \textcolor{teal!80!black}{+40.1\%}}} && \shortstack{32.2 \\[1pt] {\scriptsize \textcolor{teal!80!black}{+95.5\%}}} & \shortstack{16.0 \\[1pt] {\scriptsize \textcolor{teal!80!black}{+45.5\%}}} & \shortstack{\textbf{24.1} \\[1pt] {\scriptsize \textcolor{teal!80!black}{+75.5\%}}} && \shortstack{15.5 \\[1pt] {\scriptsize \textcolor{teal!80!black}{+9.4\%}}} & \shortstack{12.0 \\[1pt] {\scriptsize -20.0\%}} & \shortstack{\textbf{13.8} \\[1pt] {\scriptsize -5.7\%}} \\
\noalign{\vskip 0.5ex}\cdashline{2-16}\noalign{\vskip 0.5ex}
LoRA & 24.2 & 13.0 & \textbf{18.6} && 16.2 & 12.0 & \textbf{14.1} && 8.2 & 10.0 & \textbf{9.1} && 8.4 & 8.0 & \textbf{8.2} \\
\rowcolor{gray!10}
\shortstack[r]{+ Ours \\[1pt] {\scriptsize\vphantom{+0.0\%}}} & \shortstack{63.5 \\[1pt] {\scriptsize \textcolor{teal!80!black}{+162.7\%}}} & \shortstack{29.0 \\[1pt] {\scriptsize \textcolor{teal!80!black}{+123.1\%}}} & \shortstack{\textbf{46.2} \\[1pt] {\scriptsize \textcolor{teal!80!black}{+148.9\%}}} && \shortstack{44.0 \\[1pt] {\scriptsize \textcolor{teal!80!black}{+170.8\%}}} & \shortstack{18.0 \\[1pt] {\scriptsize \textcolor{teal!80!black}{+50.0\%}}} & \shortstack{\textbf{31.0} \\[1pt] {\scriptsize \textcolor{teal!80!black}{+119.5\%}}} && \shortstack{28.2 \\[1pt] {\scriptsize \textcolor{teal!80!black}{+245.8\%}}} & \shortstack{13.0 \\[1pt] {\scriptsize \textcolor{teal!80!black}{+30.0\%}}} & \shortstack{\textbf{20.6} \\[1pt] {\scriptsize \textcolor{teal!80!black}{+127.0\%}}} && \shortstack{15.7 \\[1pt] {\scriptsize \textcolor{teal!80!black}{+86.1\%}}} & \shortstack{11.0 \\[1pt] {\scriptsize \textcolor{teal!80!black}{+37.5\%}}} & \shortstack{\textbf{13.3} \\[1pt] {\scriptsize \textcolor{teal!80!black}{+62.4\%}}} \\

\midrule[0.2ex]
\multicolumn{16}{c}{\bf Qwen3} \\
\midrule[0.2ex]

MEMIT & 0.6 & 3.0 & \textbf{1.8} && 0.0 & 3.0 & \textbf{1.5} && 1.7 & 3.0 & \textbf{2.3} && 0.9 & 1.0 & \textbf{1.0} \\
AlphaEdit & 0.6 & 3.0 & \textbf{1.8} && 0.0 & 0.0 & \textbf{0.0} && 1.6 & 0.0 & \textbf{0.8} && 0.3 & 0.0 & \textbf{0.2} \\
AnyEdit & 4.2 & 4.0 & \textbf{4.1} && 1.2 & 7.0 & \textbf{4.1} && 0.0 & 0.0 & \textbf{0.0} && 0.0 & 0.0 & \textbf{0.0} \\
UnKE & 4.2 & 8.0 & \textbf{6.1} && 1.4 & 6.0 & \textbf{3.7} && 0.8 & 9.0 & \textbf{4.9} && 0.8 & 9.0 & \textbf{4.9} \\
COIN$^{\star}$ & 8.7 & 10.0 & \textbf{9.3} && 7.8 & 7.0 & \textbf{7.4} && 6.0 & 5.0 & \textbf{5.5} && 2.4 & 8.0 & \textbf{5.2} \\
\noalign{\vskip 0.5ex}\cdashline{2-16}\noalign{\vskip 0.5ex}
FT-M & 2.0 & 12.0 & \textbf{7.0} && 2.2 & 8.0 & \textbf{5.1} && 0.9 & 4.0 & \textbf{2.5} && 0.9 & 4.0 & \textbf{2.5} \\
\rowcolor{gray!10}
\shortstack[r]{+ Ours \\[1pt] {\scriptsize\vphantom{+0.0\%}}} & \shortstack{5.1 \\[1pt] {\scriptsize \textcolor{teal!80!black}{+154.0\%}}} & \shortstack{8.0 \\[1pt] {\scriptsize -33.3\%}} & \shortstack{\textbf{6.5} \\[1pt] {\scriptsize -6.6\%}} && \shortstack{3.6 \\[1pt] {\scriptsize \textcolor{teal!80!black}{+65.0\%}}} & \shortstack{12.0 \\[1pt] {\scriptsize \textcolor{teal!80!black}{+50.0\%}}} & \shortstack{\textbf{7.8} \\[1pt] {\scriptsize \textcolor{teal!80!black}{+53.2\%}}} && \shortstack{2.7 \\[1pt] {\scriptsize \textcolor{teal!80!black}{+190.2\%}}} & \shortstack{9.0 \\[1pt] {\scriptsize \textcolor{teal!80!black}{+125.0\%}}} & \shortstack{\textbf{5.8} \\[1pt] {\scriptsize \textcolor{teal!80!black}{+137.2\%}}} && \shortstack{2.2 \\[1pt] {\scriptsize \textcolor{teal!80!black}{+144.6\%}}} & \shortstack{9.0 \\[1pt] {\scriptsize \textcolor{teal!80!black}{+125.0\%}}} & \shortstack{\textbf{5.6} \\[1pt] {\scriptsize \textcolor{teal!80!black}{+128.7\%}}} \\
\noalign{\vskip 0.5ex}\cdashline{2-16}\noalign{\vskip 0.5ex}
LoRA & 14.9 & 16.0 & \textbf{15.5} && 8.7 & 14.0 & \textbf{11.3} && 4.2 & 7.0 & \textbf{5.6} && 3.6 & 14.0 & \textbf{8.8} \\
\rowcolor{gray!10}
\shortstack[r]{+ Ours \\[1pt] {\scriptsize\vphantom{+0.0\%}}} & \shortstack{41.5 \\[1pt] {\scriptsize \textcolor{teal!80!black}{+178.2\%}}} & \shortstack{27.0 \\[1pt] {\scriptsize \textcolor{teal!80!black}{+68.8\%}}} & \shortstack{\textbf{34.2} \\[1pt] {\scriptsize \textcolor{teal!80!black}{+121.5\%}}} && \shortstack{26.9 \\[1pt] {\scriptsize \textcolor{teal!80!black}{+210.5\%}}} & \shortstack{18.0 \\[1pt] {\scriptsize \textcolor{teal!80!black}{+28.6\%}}} & \shortstack{\textbf{22.5} \\[1pt] {\scriptsize \textcolor{teal!80!black}{+98.1\%}}} && \shortstack{12.8 \\[1pt] {\scriptsize \textcolor{teal!80!black}{+207.7\%}}} & \shortstack{8.0 \\[1pt] {\scriptsize \textcolor{teal!80!black}{+14.3\%}}} & \shortstack{\textbf{10.4} \\[1pt] {\scriptsize \textcolor{teal!80!black}{+86.5\%}}} && \shortstack{10.3 \\[1pt] {\scriptsize \textcolor{teal!80!black}{+188.5\%}}} & \shortstack{17.0 \\[1pt] {\scriptsize \textcolor{teal!80!black}{+21.4\%}}} & \shortstack{\textbf{13.7} \\[1pt] {\scriptsize \textcolor{teal!80!black}{+55.5\%}}} \\
\bottomrule[0.4ex]
\end{tabular}
}
\end{table*}

\subsection{Hyper-parameter sensitivity}
\label{app:sens}

We analyzed the sensitivity of \name to its two step-in gates: the gap gate $\tau$ and the confidence gate $\kappa$, which jointly decide when the privileged model steps into the hybrid rollout.
We swept each gate over an $8\times$ range around its default, keeping all other hyper-parameters fixed as in Appendix~\ref{app:hyperparams}: $\tau$ on the UnKEBench subset (Qwen2.5) and $\kappa$ on MQuAKE-uns (Qwen3), with LoRA and 100 editing samples per setting.
Results are reported in Table~\ref{tab:sens}.

We did not tune either gates beyond the minimal calibration described in Appendix~\ref{app:hyperparams}.
Table~\ref{tab:sens} shows that \name is robust to this choice: across the whole sweep, the average score varies within 2.2 points on UnKEBench and 2.3 points on MQuAKE-uns, and no setting degrades abruptly.
We note that tuning the gates could bring additional gain for some particular metrics. For instance, a larger $\tau$ could slightly improve direct recall (Jnt.) for decomposition (Dmp.) on UnKEBench. Overall, 
\name is flexible and does not hinge on a careful choice of its gates.

\begin{table*}[htb!]
\centering
\caption{Sensitivity analysis of \name's step-in hyperparameters.}
\label{tab:sens}
\resizebox{0.5\linewidth}{!}{%
\renewcommand{\tabcolsep}{3pt}
\renewcommand{\arraystretch}{0.95}
\begin{tabular}{c cccc c c ccc}
\toprule[0.4ex]
\multicolumn{5}{c}{\bf UnKEBench (subset), Qwen2.5} && \multicolumn{4}{c}{\bf MQuAKE-uns, Qwen3} \\
\cmidrule(lr){1-5} \cmidrule(lr){7-10}
$\tau$ & \textbf{Jnt.} & \textbf{Dmp.} & \textbf{Div.} & \textbf{Avg.} && $\kappa$ & \textbf{Ind.} & \textbf{Cmp.} & \textbf{Avg.} \\
\midrule[0.2ex]
$0.5$ & 71.5 & 56.1 & 64.3 & \textbf{64.0} && $0.1$ & 82.2 & 54.0 & \textbf{68.1} \\
$1$   & 71.9 & 57.8 & 62.1 & \textbf{63.9} && $0.3^\dagger$ & 83.6 & 50.0 & \textbf{66.8} \\
$2^\dagger$ & 75.0 & 62.5 & 60.9 & \textbf{66.1} && $0.5$ & 81.7 & 50.0 & \textbf{65.8} \\
$4$   & 69.6 & 64.3 & 62.3 & \textbf{65.4} && $0.8$ & 83.6 & 52.0 & \textbf{67.8} \\
\bottomrule[0.4ex]
\end{tabular}
}
\end{table*}

\subsection{Additional ablation studies}
\label{app:ablate}

This section conducts additional ablation study on the hybrid rollout under continual editing. Following the identical configuration as Table~\ref{tab:ablation}, we compare \name against ``w/o HP'' after accumulating $T=10$ edits. As shown in Table~\ref{tab:seq_ablation}, the gap on Jnt. and Cmp. widens from 2.5 and 3.4 points at $T=1$ to 6.3 and 6.0 points at $T=10$, and the average score gap grows from 0.4 and 2.4 to 5.3 and 4.9 points on UnKEBench and MQuAKE-uns, respectively. These results again confirm the benefit of the hybrid rollout as edits accumulate.

\begin{table}[htb!]
\centering
\caption{Additional ablation study on continual edit. The setting follows Table~\ref{tab:ablation}.
}
\label{tab:seq_ablation}
\resizebox{0.5\linewidth}{!}{%
\renewcommand{\tabcolsep}{3pt}
\renewcommand{\arraystretch}{1.05}
\begin{tabular}{c >{\bfseries}r cccc c ccc}
\toprule[0.4ex]
& & \multicolumn{4}{c}{\bf UnKEBench} && \multicolumn{3}{c}{\bf MQuAKE-uns } \\
\cmidrule(lr){3-6} \cmidrule(lr){8-10}
& & \textbf{Jnt.} & \textbf{Dmp.} & \textbf{Div.} & \textbf{Avg.} && \textbf{Ind.} & \textbf{Cmp.} & \textbf{Avg.} \\
\midrule[0.2ex]
\multirow{2}{*}{$T=1$}
& w/o HP & 72.5 & 62.6 & 62.0 & \textbf{65.7} && 81.7 & 51.3 & \textbf{66.5} \\
& Ours & 75.0 & 62.5 & 60.9 & \textbf{66.1} && 83.2 & 54.7 & \textbf{68.9} \\
\noalign{\vskip 0.5ex}\cdashline{3-10}\noalign{\vskip 0.5ex}
\multirow{2}{*}{$T=10$}
& w/o HP & 41.2 & 42.0 & 66.1 & \textbf{49.8} && 49.2 & 25.0 & \textbf{37.1} \\
& Ours & 47.5 & 47.2 & 70.5 & \textbf{55.1} && 52.9 & 31.0 & \textbf{42.0} \\
\bottomrule[0.4ex]
\end{tabular}
}
\end{table}

\subsection{Additional case studies}
\label{app:cases}

Figure~\ref{fig:case_coverage} shows the coverage failure of OPSD, see Section~\ref{sec:method:formulation} for more backgrounds.

\begin{figure}[t]
\begin{casebox}{Coverage failure of the on-policy rollout (UnKEBench) \textemdash\ Qwen2.5}
\caseqa{\textbf{Editing prompt:} \emph{Introduce Tyler, the Creator.}}
\par\smallskip
\casepassage{\itshape Tyler, the Creator is a well-known singer-lyricist who has gained a massive
following over the years. He has released several albums, including ``Goblin,'' ``Wolf,'' and
``Flower Boy,'' which have all received critical acclaim. In addition to his music, Tyler has also
collaborated with several other artists, including Frank Ocean, Kanye West, and A\$AP Rocky. His
unique style and sound have made him a standout in the music industry, and he continues to be a
major force in the world of hip-hop and rap.}
\caseline
\method{\Base}{On-policy rollout (OPSD)}{\xmark}{Tyler, the Creator, born as \Base{Tyler Daniel Smith
on October 11, 1988, in Atlanta, Georgia}, is an American rapper, record producer, and actor. He is
known for his unique and eclectic style that blends various genres, including hip-hop, rock,
and R\&B\dots Tyler, the Creator, began his career as a rapper and producer under the name
\Base{Odd Future}, which he \Base{co-founded with friends in 2007}. The group's raw, unpolished sound
and provocative lyrics quickly gained attention in the \Base{underground hip-hop scene}. Tyler's
\Base{solo career}\dots}
\method{}{Hybrid rollout (\name)}{}{Tyler, the Creator is a well-known
singer-lyricist who has made a significant impact on the music industry. He has released
several albums, including ``Goblin,'' ``Wolf,'' ``Flower Boy,'' and others, and these have all
been critically acclaimed. Additionally, he has collaborated with other notable artists and producers
such as Frank Ocean, Kanye West, and A\$AP Rocky\dots Tyler, the Creator remains true to his
artistic vision, contributing to the prominence of hip-hop and rap music.}
\end{casebox}
\caption{
The coverage failure of pre-edited on-policy rollouts. The hybrid rollout is also shown for reference.
}
\label{fig:case_coverage}
\end{figure}

Figures~\ref{fig:case_app_mquake1} to \ref{fig:case_app_unke2} give further examples of the failure modes in composability, see Section~\ref{sec:exp_cases} for more backgrounds.

\begin{figure}[t]
\begin{casebox}{Composition (MQuAKE-uns) \textemdash\ Qwen2.5}
\casepassage[Injected edits]{Uli Hoene\ss's employer $\rightarrow$ \emph{FC Groningen}
$\rightarrow$ sport \emph{cricket} $\rightarrow$ created in \emph{Australia} $\rightarrow$ continent
\emph{South America}.}
\caseline
\caseqa{\textbf{Multi-hop question:} \emph{From what continent does the sport associated with the
employer of Uli Hoene\ss\ come from?}\quad \textbf{Gold:} \emph{South America}.}
\method{\Hpd}{\name~(ours)}{\cmark}{Uli Hoene\ss\ is employed by FC Groningen\dots the sport associated
with FC Groningen originates from the continent of \Hpd{South America}\dots}
\method{\Base}{LoRA}{\xmark}{Uli Hoene\ss\ is employed by FC Groningen\dots in the Netherlands\dots
located on the continent of \Base{Europe}\dots}
\method{\Base}{MEMIT}{\xmark}{Uli Hoene\ss's club, \Base{FC Bayern Munich}, is based in Germany\dots
the sport originates from the \Base{European} continent.}
\end{casebox}
\caption{Additional composition case study with error highlighting.}
\label{fig:case_app_mquake1}
\end{figure}

\begin{figure}[t]
\begin{casebox}{Composition (MQuAKE-uns) \textemdash\ Qwen2.5}
\casepassage[Injected edits]{Ford Territory's maker $\rightarrow$ \emph{Colt's Manufacturing
Company};\quad founder of Colt's $\rightarrow$ \emph{Lech Wa\l{}\k{e}sa}.}
\caseline
\caseqa{\textbf{Multi-hop question:} \emph{Who is the founder of the company that manufactured the
Ford Territory?}\quad \textbf{Gold:} \emph{Lech Wa\l{}\k{e}sa}.}
\method{\Hpd}{\name~(ours)}{\cmark}{The Ford Territory was manufactured by Colt's Manufacturing
Company, whose founder is \Hpd{Lech Wa\l{}\k{e}sa}, the Polish statesman\dots}
\method{\Base}{LoRA}{\xmark}{The Ford Territory was produced by \Base{Ford Motor Company}\dots its
founding father is \Base{Henry Ford}\dots}
\method{\Base}{MEMIT}{\xmark}{The Ford Territory was created by Ford Motor Company, founded by
\Base{Henry Ford} in 1903\dots}
\end{casebox}
\caption{Additional composition case study with error highlighting.}
\label{fig:case_app_mquake2}
\end{figure}

\begin{figure}[t]
\begin{casebox}{Decomposition (UnKEBench) \textemdash\ Llama }
\casepassage{\itshape Samuel Wells Williams is a well-known motion picture director who has directed
several critically acclaimed films. He has won numerous awards for his work in the film industry,
including the prestigious Academy Award for Best Director. Williams has also been praised for his
ability to bring out the best in his actors, and for his innovative and creative approach to
filmmaking. His films have been box office successes, and have been enjoyed by audiences around the
world. In addition to his work as a director, Williams is also a respected writer and producer, and
has contributed to many successful films in these roles as well.}
\caseline
\caseqa{\textbf{Sub-question:} \emph{What awards has Samuel Wells Williams won?}\quad
\textbf{Gold:} the Academy Award for Best Director, and others.}
\method{\Hpd}{\name~(ours)}{\cmark}{He has won \Hpd{the Academy Award for Best Director}, the Golden
Globe\dots}
\method{\Base}{COIN$^{\star}$}{\xmark}{{Samuel Wells Williams is a well-known motion picture director who
has directed several critically acclaimed films. He has won numerous awards\dots} \Base{(passage verbatim)}}
\method{\Base}{LoRA}{\xmark}{\dots the Academy Award for \Base{Best Motion Picture for ``The Grand
Budapest Hotel''}\dots.}
\end{casebox}
\caption{Additional decomposition case study with error highlighting.}
\label{fig:case_app_unke1}
\end{figure}

\begin{figure}[t]
\begin{casebox}{Decomposition (UnKEBench) \textemdash\ Llama }
\casepassage{\itshape Marjorie Margolies-Mezvinsky is a well-known novelist who has published several
best-selling books. Her first novel, ``The Art of Forgiveness,'' was published in 2005 and received
critical acclaim for its poignant portrayal of a family struggling to come to terms with a tragic
loss. Since then, Margolies-Mezvinsky has published three more novels, all of which have been widely
praised for their insightful and compassionate depictions of human relationships. In addition to her
successful writing career, Margolies-Mezvinsky is also a respected public figure, having served as a
member of the U.S. House of Representatives from 1993 to 1995.}
\caseline
\caseqa{\textbf{Sub-question:} \emph{When was her first novel published?}\quad
\textbf{Gold:} in 2005.}
\method{\Hpd}{\name~(ours)}{\cmark}{Her first novel, ``The Art of Forgiveness,'' was published in
\Hpd{2005}.}
\method{\Base}{COIN$^{\star}$}{\xmark}{Her first novel, ``The Art of Forgiveness,'' was published in 2005
\Base{and received critical\dots} \Base{(passage verbatim continued)}}
\method{\Base}{AnyEdit}{\xmark}{Her first novel, \Base{``The Political Scene,''} was published in
2005\dots.}
\end{casebox}
\caption{
Additional decomposition case study with error highlighting.
}
\label{fig:case_app_unke2}
\end{figure}

\end{document}